\documentclass{article}
\usepackage{ijcai25}

\usepackage{times}
\usepackage{soul}
\usepackage{url}
\usepackage[hidelinks]{hyperref}
\usepackage[utf8]{inputenc}
\usepackage[small]{caption}
\usepackage{graphicx}
\usepackage{amsmath}
\usepackage{amsthm}
\usepackage{booktabs}
\usepackage{algorithm}
\usepackage{algorithmic}
\usepackage[switch]{lineno}

\usepackage{threeparttable}
\usepackage{subfigure}
\usepackage{amssymb}
\usepackage{color}
\usepackage{multirow}

\newcommand{\alg}{Dyn-D$^2$P}
\newcommand{\algbase}{Const-D$^2$P}
\newcommand{\algS}{Dyn[$C$]-D$^2$P}
\newcommand{\algmu}{Dyn[$\mu$]-D$^2$P}

\definecolor{mydarkgreen}{RGB}{20, 201, 0}
\definecolor{purple}{RGB}{128, 0, 128}

\newtheorem{Thm}{Theorem}
\newtheorem{Def}{Definition}
\newtheorem{Lem}{Lemma}
\newtheorem{Ass}{Assumption}
\newtheorem{Rem}{Remark}
\newtheorem{Cor}{Corollary}
\newtheorem{Pro}{Proposition}

\title{Dyn-D$^2$P: Dynamic Differentially Private Decentralized Learning \\with Provable Utility Guarantee}

\author{
Zehan Zhu$^1$
\and
Yan Huang$^1$
\and
Xin Wang$^2$
\and
Shouling Ji$^1$
\And
Jinming Xu$^{1}$\thanks{Corresponding author.}
\affiliations
$^1$Zhejiang University, Hangzhou, China\\
$^2$Qilu University of Technology, Jinan, China
\emails
\{12032045, huangyan5616\}@zju.edu.cn,
xinwang@qlu.edu.cn,
\{sji, jimmyxu\}@zju.edu.cn
}

\begin{document}
\maketitle

\begin{abstract}
Most existing decentralized learning methods with differential privacy (DP) guarantee rely on constant gradient clipping bounds and fixed-level DP Gaussian noises for each node throughout the training process, leading to a significant accuracy degradation compared to non-private counterparts. In this paper, we propose a new \underline{Dyn}amic \underline{D}ifferentially \underline{P}rivate \underline{D}ecentralized learning approach (termed {\alg}) tailored for general time-varying directed networks. Leveraging the Gaussian DP (GDP) framework for privacy accounting, {\alg} dynamically adjusts gradient clipping bounds and noise levels based on gradient convergence. This proposed dynamic noise strategy enables us to enhance model accuracy while preserving the total privacy budget. Extensive experiments on benchmark datasets demonstrate the superiority of {\alg} over its counterparts employing fixed-level noises, especially under strong privacy guarantees. Furthermore, we provide a provable utility bound for {\alg} that establishes an explicit dependency on network-related parameters, with a scaling factor of $1/\sqrt{n}$ in terms of the number of nodes $n$ up to a bias error term induced by gradient clipping. To our knowledge, this is the first model utility analysis for differentially private decentralized non-convex optimization with dynamic gradient clipping bounds and noise levels.
\end{abstract}

\section{Introduction}
\label{Intro}
Distributed learning has recently attracted significant attention due to its great potential in enhancing computing efficiency and has thus been widely adopted in various application domains~\cite{langer2020distributed}.
In particular, it can be typically modeled as a non-convex finite-sum optimization problem solved by a group of $n$ nodes as follows:
\begin{equation}
\label{global_loss_function}
\underset{x\in \mathbb{R}^d}{\min}f\left( x \right) \triangleq \frac{1}{n}\sum_{i=1}^n{f_i\left( x \right)}
\end{equation}
with $f_i\left( x \right) = \frac{1}{J}\sum_{j=1}^J{f_i\left( x;j \right)}$, where $J$ denotes the local dataset size of each node, $f_i(x; j)$ denotes the loss function of the $j$-th data sample at node $i$ with respect to the model parameter $x\in \mathbb{R}^d$, and $f_i\left( x \right)$ and $f\left( x \right)$ denote the local objective function at node $i$ and the global objective function, respectively. All nodes collaborate to seek the optimal model parameter to minimize the global loss $f\left( x \right)$, and each node $i$ can only evaluate local stochastic gradient $\nabla f_i\left( x;\xi_i \right)$ where $\xi_i \in \{1,2,...,J\}$ is a randomly chosen sample.

Bottlenecks such as high communication overhead and the vulnerability of central nodes in parameter server-based methods~\cite{zinkevich2010parallelized,mcmahan2017communication} motivate researchers to investigate fully decentralized methods~\cite{lian2017can,tang2018d,koloskova2019decentralized} to solve Problem~\eqref{global_loss_function}, where a central node is not required and each node only communicates with its neighbors. 
We thus consider such a fully decentralized setting in this paper, with a particular focus on general and practical time-varying directed communication networks for communication among nodes. Decentralized learning involves each node performing local stochastic gradient descent to update its model parameters, followed by communication with neighboring nodes to share and mix model parameters before proceeding to the next iteration~\cite{zhu2024r}. However, the frequent information exchange among nodes poses significant privacy concerns, as the exposure of model parameters could potentially be exploited to compromise the privacy of original data samples~\cite{wang2019beyond}. To protect each node from these potential attacks, differential privacy (DP), as a theoretical tool to provide rigorous privacy guarantees and quantify privacy loss, can be integrated into each node within decentralized learning to enhance privacy protection~\cite{cheng2018leasgd,yu2021decentralized}.

Existing decentralized learning algorithms with differential privacy guarantee for non-convex problems tend to employ a constant/fixed gradient clipping bound $\bar{C}$~\cite{yu2021decentralized,xu2021dp,li2025convergence} to estimate the $l_2$ sensitivity of gradient update and uniformly distribute privacy budgets across all iterations. As a result, each node injects fixed-level DP Gaussian noises with a variance proportional to the estimated sensitivity (i.e., constant clipping bound $\bar{C}$) before performing local SGD at each iteration. However, our empirical observations indicate that the norm of gradient typically decays as training progresses and ultimately converges to a small value (c.f., Figure~\ref{demo_gradient_convergence}). This observation suggests that using a constant clipping bound to estimate $l_2$ sensitivity throughout the training process may be overly conservative, as gradient norms are often smaller than the constant $\bar{C}$, especially in the later stages of training. Therefore, the added fixed-level Gaussian noise becomes unnecessary and instead degrades the model accuracy without providing additional privacy benefits. The following question thus arises naturally: 
\textit{Can we design a decentralized learning method that dynamically adjusts the level of DP noise during the training process to minimize  accuracy loss while maintaining privacy guarantee?}

\begin{figure}[t]
\centering
\includegraphics[width=0.6\linewidth]{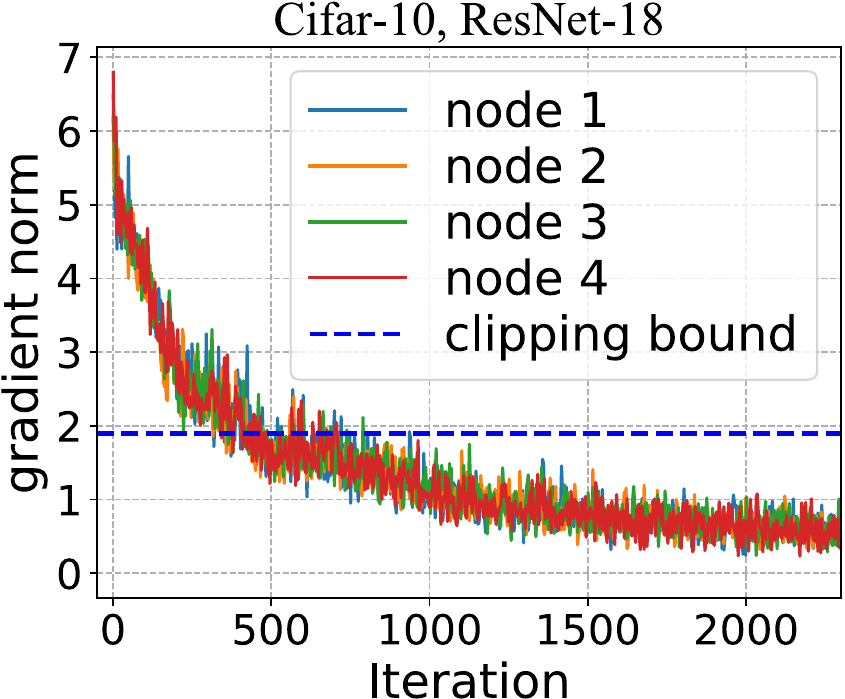}
\caption{The evolution of the gradient norm when training ResNet-18 on Cifar-10 dataset in a fully decentralized setting with 4 nodes. It can be observed that the stochastic gradient norm at each node typically decays as training progresses, and eventually falls below the constant clipping bound (indicated by the blue dotted horizontal line) after certain number of iterations, which makes the clipping operation ineffective in the later stages of training.}
\label{demo_gradient_convergence}
\end{figure}

To this end, we develop a new dynamic differentially private learning method for solving Problem~\eqref{global_loss_function} in fully decentralized settings, which enhances model accuracy while adhering to a total privacy budget constraint. The main contributions of this work are threefold: 
\begin{itemize}
    \item We propose a differentially private decentralized learning method with a dynamic DP Gaussian noise strategy (termed {\alg}), tailored for general time-varying directed networks. In particular, each node adds noise with a variance calibrated by a dynamically decaying gradient clipping bound and an increasing per-step privacy budget appropriately allocated across iterations. This mechanism enables each node to apply dynamically decreasing noise, thereby enhancing model accuracy without compromising the total privacy budget.
    
    \item Theoretically, we investigate the impact of dynamic noise strategy on model utility for a general form of {\alg} where clipping bounds and noise levels can be arbitrary sequences (c.f., Theorem~\ref{main_Theorem}), revealing the advantages of using dynamically decaying clipping bounds (c.f., Remark~\ref{main_remark}). By employing exponentially decaying sequences, we prove the utility bound of Dyn-D$^2$P with explicit dependency on the network-related parameter, exhibiting a scaling factor of $1/\sqrt{n}$ in terms of the number of nodes up to a bias error term induced by gradient clipping (c.f., Corollary~\ref{Corollary_after_theorem}). To our knowledge, this is the first provable utility guarantee in the realm of dynamic differentially private decentralized learning.
    
    \item Extensive experiments are conducted to verify the performance of the proposed {\alg}. The results show that, under the same total privacy budget, {\alg} achieves superior accuracy compared to its counterpart using fixed-level DP Gaussian noise, especially under strong privacy guarantees. Moreover, we validate the robustness of {\alg} against certain hyper-parameters related to the varying rates of gradient clipping bound and per-step privacy budget, and verify the  performance of {\alg} over different graphs and node numbers, which aligns with our theoretical findings.
\end{itemize}

\section{Preliminary and Related Work}
\textbf{Differential privacy.}
Differential privacy was originally introduced in the seminal work by Dwork \textit{et al.}~\shortcite{dwork2006our} as a foundational concept for quantifying the privacy-preserving capabilities of randomized algorithms, and has now found widespread applications in a variety of domains that necessitate safeguarding against unintended information leakage~\cite{li2019differentially,shin2018privacy,wei2021low}.
We recall the standard definition of DP as follows.
\begin{Def}[$(\epsilon,\delta)$-DP~\cite{dwork2014algorithmic}]
A randomized mechanism $\mathcal{M}$ with domain $\mathcal{D}$ and range $\mathcal{R}$ satisfies $(\epsilon,\delta)$-differential privacy (or $(\epsilon,\delta)$-DP), if for any two adjacent inputs $\mathrm{x},\mathrm{x}^{\prime}\in \mathcal{D}$ differing on a single entry and for any subset of outputs $O\subseteq \mathcal{R}$, it holds that
\begin{equation}
Pr\left[ \mathcal{M}\left( \mathrm{x} \right) \in O \right] \leqslant e^{\epsilon}Pr\left[ \mathcal{M}\left( \mathrm{x}^{\prime} \right) \in O \right] +\delta ,
\end{equation}
where the privacy budget $\epsilon$ denotes the privacy lower bound to
measure a randomized query and $\delta$ is the probability of breaking this bound. Note that a smaller value of $\epsilon$ implies a stronger privacy guarantee.
\end{Def}
The following proposition provides the Gaussian DP (GDP) mechanism to ensure privacy guarantee.
\begin{Pro}[$\mu$-GDP~\cite{dong2019gaussian}]
\label{formal_def_of_mu}
Let $f:\mathcal{D} \rightarrow \mathbb{R}^d$ be a function and $S$ be its $l_2$ sensitivity. Then, drawing a random noise $N$ from Gaussian distribution with $N\thicksim \mathcal{N}\left( 0,\sigma ^2\mathbb{I}_d \right) $ and adding it to $f$ such that $\mathcal{M}(\mathrm{x})=f(\mathrm{x})+N$ satisfies $\mu$-GDP if $\sigma$ is set as
\begin{equation}
\label{def_of_mu}
\sigma =S/\mu,
\end{equation}
where $\mu$ is the privacy budget in the GDP framework, and a smaller value of $\mu$ implies a stronger privacy guarantee.
\end{Pro}

The above proposition shows that the variance $\sigma$ of the added noise required to ensure $\mu$-GDP is dependent on both privacy budget $\mu$ and sensitivity $S$. It should be noted that the above privacy guarantee in the sense of GDP can be transformed into the standard DP by the following proposition which shows that there is a one-to-one correspondence between $\epsilon$ and $\mu$ values when fixing $\delta$. 
\begin{Pro}[From $\mu$-GDP to $\left( \epsilon, \delta \right)$-DP~\cite{bu2020deep}]
\label{pro_of_privacy_tranfer}
A random mechanism is $\mu$-GDP if and only if it is $\left( \epsilon ,\delta \right) $-DP for all $\epsilon \geqslant 0$, where
\begin{equation}
\label{privacy_tranfer}
\delta  =\Phi \left( -\frac{\epsilon}{\mu}+\frac{\mu}{2} \right) -e^{\epsilon}\Phi \left( -\frac{\epsilon}{\mu}-\frac{\mu}{2} \right) ,
\end{equation}
and $\Phi \left( \cdot \right) $ is the Gaussian cumulative distribution function.
\end{Pro}

In what follows, we will review existing works related to achieving DP guarantees in machine learning and highlight the limitations in decentralized scenarios.

\textbf{Decentralized learning methods with DP guarantee.}
DP guarantee is initially integrated into a centralized (single-node) setting for designing differentially private stochastic learning algorithms~\cite{abadi2016deep,wang2017differentially,iyengar2019towards,chen2020understanding,wang2020differentially}.
Further, DP guarantee is considered in distributed learning with server-client structures and the representative works include~\cite{mcmahan2017learning,li2019asynchronous,wang2019efficient,wu2020value,wei2020federated,zeng2021differentially,wei2021user,li2022soteriafl,liu2022loss,zhou2023optimizing,wei2023securing}.
Recently, there have been efforts to achieve DP guarantees for fully decentralized learning algorithms. 
For example,
Cheng \textit{et al.}~\shortcite{cheng2018leasgd,cheng2019towards} achieve DP in fully decentralized learning for only strongly convex problems. 
Wang and Nedic~\shortcite{wang2022tailoring} achieve DP in fully decentralized architectures by tailoring gradient methods for deterministic optimization problems. 
For non-convex stochastic optimization problems as we consider in this work, Yu \textit{et al.}~\shortcite{yu2021decentralized} present a differentially private decentralized learning method (DP$^2$-SGD) based on D-PSGD~\cite{lian2017can}, which relies on a fixed communication topology and uses the basic composition theorem to bound the overall privacy loss.
To have a tight privacy guarantee, Xu \textit{et al.}~\shortcite{xu2021dp} propose a differentially private asynchronous decentralized learning method (A(DP)$^{2}$SGD) based on AD-PSGD~\cite{lian2018asynchronous}, which provides privacy guarantee in the sense of R\'enyi differential privacy (RDP)~\cite{mironov2017renyi}. 
However, it should be noted that the aforementioned two algorithms~\cite{yu2021decentralized,xu2021dp} work only for undirected communication graphs which is often not satisfied in practical scenarios.
Most recently, Li and Chi~\shortcite{li2025convergence} achieve DP guarantee as well as communication compression in decentralized learning for non-convex problems with the total privacy cost calculated via the moments accountant technique~\cite{abadi2016deep}, while their methods are only applicable to time-invariant graphs.

\textbf{Learning with dynamic DP Gaussian noise levels.} 
For the aforementioned differentially private decentralized methods designed for non-convex stochastic optimization problems~\cite{yu2021decentralized,xu2021dp,li2025convergence}, the injected fixed-level noise may exceed what is actually needed for privacy requirements as training progresses, especially during the later stages of training, since their estimated sensitivity based on constant/fixed gradient clipping bound $\bar{C}$ may not reflect the actual value of sensitivity (c.f., Figure~\ref{demo_gradient_convergence}). The overestimate of sensitivity may, indeed, lead to a waste of unnecessary privacy budget during the training process~\cite{wei2023securing}.
Therefore, a tighter sensitivity estimate is useful for improving model accuracy without sacrificing privacy.
There have been few works dedicated to tightly estimate the sensitivity in a dynamic manner. For instance, a scheme of decaying gradient clipping bounds has been employed to estimate the sensitivity in differentially private centralized learning~\cite{du2021dynamic,wei2021gradient}, resulting in a decreasing amount of noise injection. In the realm of distributed learning, a similar strategy of dynamic clipping bounds is utilized in~\cite{andrew2021differentially} to estimate sensitivity.
Most recently, Wei \textit{et al.}~\shortcite{wei2023securing} use the minimum of decaying clipping bounds and current gradient norms to estimate the $l_2$ sensitivity, leading to a less amount of noise injection.
However, these distributed methods~\cite{andrew2021differentially,fu2022adap,wei2023securing} only focus on server-client architecture and, most importantly, they do not provide any theoretical guarantee on model utility. In this paper, we aim to design a differentially private decentralized learning method that incorporates dynamic noise strategies in fully decentralized settings, and provide rigorous theoretical guarantee on model utility, as well as its utility-privacy trade-off.

\section{Algorithm Development}
\label{algorithm_develop}
In this section, we develop our differentially private decentralized learning methods using the Gaussian DP (GDP) framework as depicted in Proposition~\ref{formal_def_of_mu}, which measures the privacy profile $\left( \epsilon,\delta \right)$ in terms of $\mu$ according to Proposition~\ref{pro_of_privacy_tranfer}. We consider solving Problem~\eqref{global_loss_function} over the following general peer-to-peer network model.

\textbf{Network model.} The communication topology is modeled as a sequence of time-varying directed graph $\mathcal{G}^k=\left( \mathcal{V},\mathcal{E}^k \right) $, where $\mathcal{V}=\{1,2,...,n\}$ denotes the set of nodes and $\mathcal{E}^k \subset \mathcal{V} \times \mathcal{V}$ denotes the set of directed edges/links at iteration $k$. We associate each graph $\mathcal{G}^k$ with a non-negative mixing matrix $P^k \in \mathbb{R}^{n \times n}$ such that $(i,j) \in \mathcal{E}^k$ if $P_{i,j}^k > 0$, i.e., node $i$ receiving a message from node $j$ at iteration $k$. We assume that each node is an in-neighbor of itself.

\begin{algorithm}[t!]
\caption{{\alg}} 
\label{Adpt_PrivSGP_combine} 
\begin{algorithmic}[1]
\STATE \textbf{Initialization:}  DP budget $(\epsilon,\delta)$, $x_{i}^{0}=z_{i}^{0}\in \mathbb{R}^d$, $w_i^0=1$, step size $\gamma > 0$, total number of iterations $K$, initial clipping bound $C_0$ and hyper-parameters $\rho_c >1$ and $\rho_\mu >1$.

\FOR{$k=0,1,...,K-1$, at node $i$,}
\STATE Randomly samples a local training data $\xi_i^k$ with the sampling probability $\frac{1}{J}$;
\STATE Computes stochastic gradient at $z_i^k$: $\nabla f_i(z_i^k;\xi_i^k)$;
\STATE Calculates the clipping bound by: $C_k=C_0\cdot \left( \rho _c \right) ^{-\frac{k}{K}}$; 
\STATE Clips the stochastic gradient: 
\begin{equation}
\label{clipped_g_i_k}
\begin{aligned}
g_{i}^{k}&=\mathrm{Clip}\left( \nabla f_i\left( z_{i}^{k};\xi _{i}^{k} \right) ;C_k \right) 
\\
&=\nabla f_i\left( z_{i}^{k};\xi _{i}^{k} \right) \cdot \min \left\{ 1,\frac{C_k}{\left\| \nabla f_i\left( z_{i}^{k};\xi _{i}^{k} \right) \right\|} \right\} ;
\end{aligned}
\end{equation}
\STATE Calculates the per-step privacy budget by: $\mu _k=\mu _0\cdot \left( \rho _{\mu} \right) ^{\frac{k}{K}}$ with $\mu_0$ the solution of~\eqref{solution_of_mu_0};
\STATE Calculates the DP noise variance by: 
\begin{equation*}
\sigma _k=\frac{C_k}{\mu _k}=\frac{C_0\cdot \left( \rho _c \right) ^{-\frac{k}{K}}}{\mu _0\cdot \left( \rho _{\mu} \right) ^{\frac{k}{K}}}=\frac{C_0}{\mu _0}\cdot \left( \rho _c\cdot \rho _{\mu} \right) ^{-\frac{k}{K}};
\end{equation*}
\STATE Draws randomized noise $N_i^k$ from the Gaussian distribution: $N_{i}^{k}\sim \mathcal{N}\left( 0,\sigma_k^2\mathbb{I}_d \right);$
\STATE Differentially private local SGD: 
\begin{equation*}
x_{i}^{k+\frac{1}{2}}=x_{i}^{k}-\gamma (g_{i}^{k}+N_{i}^{k}) ;
\end{equation*}
\STATE Sends $\left( x_i^{k+\frac{1}{2}}, w_i^k \right)$ to all out-neighbors and receives $\left( x_j^{k+\frac{1}{2}}, w_j^k \right)$ from all in-neighbors ;
\STATE Updates $x_{i}^{k+1}$ by: $x_{i}^{k+1}=\sum_{j=1}^n{P_{i,j}^{k}x_{j}^{k+\frac{1}{2}}}$ ;
\STATE Updates $w_{i}^{k+1}$ by:
$w_{i}^{k+1}=\sum_{j=1}^n{P_{i,j}^{k}w_{j}^{k}}$;
\STATE Updates $z_{i}^{k+1}$ by:
$z_{i}^{k+1}=x_{i}^{k+1}/w_i^{k+1}$.
\ENDFOR
\end{algorithmic}
\end{algorithm}

The following assumptions are made on the mixing matrix and graph for the above network model to facilitate the subsequent utility analysis for our proposed algorithm.
\begin{Ass}[Mixing matrix]
\label{assumption_mixing_matrix}
The non-negative mixing matrix $P^k, \forall k$ is column-stochastic, i.e., $\mathbf{1}^\top P^k=\mathbf{1}^\top$, where $\mathbf{1}$ is an all-one vector.
\end{Ass}

\begin{Ass}[$B$-strongly connected]
\label{assumption_graph}
There exist positive integers $B$ and $\bigtriangleup$ such that the graph formed by the edge set $\bigcup\nolimits_{k=lB}^{\left( l+1 \right) B-1}{\mathcal{E}^k}$ is strongly connected and has a diameter of at most $\bigtriangleup$ for $\forall l \geqslant 0$.
\end{Ass}

Now, we present our differentially private decentralized learning algorithm (termed {\alg}) with a dynamic noise strategy, which works over the above general network model. The complete pseudocode is summarized in Algorithm~\ref{Adpt_PrivSGP_combine}.
At a high level, {\alg} is comprised of local SGD and the averaging of neighboring information, following a framework similar to SGP~\cite{assran2019stochastic} which employs the Push-Sum protocol~\cite{kempe2003gossip} to tackle the unblanceness of directed graphs. However, the key distinction lies in the gradient clipping operation and the injection of DP Gaussian noise before performing local SGD.
In particular, each node $i$ maintains three variables during the learning process: i) the model parameter $x_i^k$; ii) the scalar Push-Sum weight $w_i^k$ and iii) the de-biased parameter $z_i^k=x_i^k/w_i^k$, with the initialization of $x_i^0=z_i^0 \in \mathbb{R}^d$ and $w_i^0=1$ for all nodes $i \in \mathcal{V}$.
At each iteration $k$, each node $i$ updates as follows:
\begin{equation*}
\begin{aligned}
&\texttt{Private local SGD:} \,\, \, x_{i}^{k+\frac{1}{2}}=x_{i}^{k}-\gamma \left( g_i^k+N_{i}^{k} \right) , 
\\
&\texttt{Averaging:} \,\, \, x_{i}^{k+1}=\sum_{j=1}^n{P_{i,j}^{k}x_{j}^{k+\frac{1}{2}}},w_{i}^{k+1}=\sum_{j=1}^n{P_{i,j}^{k}w_{j}^{k}},
\\
&\texttt{De-bias:}\,\, \,  z_{i}^{k+1}=x_{i}^{k+1}/w_{i}^{k+1},
\end{aligned}
\end{equation*}
where $\gamma > 0$ is the step size and $g_i^k$ is the clipped gradient (c.f.,~\eqref{clipped_g_i_k}). $N_i^k$ denotes the injected random noise to ensure $\mu_k$-GDP guarantee for node $i$ at iteration $k$. This noise is drawn from a Gaussian distribution with variance $\sigma_k^2$, calibrated according to the dynamic clipping bound $C_k$ and per-step privacy budget $\mu_k$.
We note that the two key mechanisms in the proposed {\alg} to achieve dynamic noise levels and enhance model accuracy include: i) dynamically decreasing clipping bounds (c.f., line 5 in Algorithm~\ref{Adpt_PrivSGP_combine}); ii) dynamically increasing per-step privacy budget (c.f., line 7 in Algorithm~\ref{Adpt_PrivSGP_combine}). The detailed design and motivation for these mechanisms will be explained as follows.

\textbf{Dynamic decreasing clipping bounds}. For differentially private decentralized learning algorithms, the gradient clipping operation is necessary for each node to bound the $l_2$ sensitivity of local SGD update and inject noise accordingly calibrated with $l_2$ sensitivity (i.e., clipping bound) and privacy budget.
According to the previous discussion in Section~\ref{Intro} that the norm of stochastic gradient of each node typically decreases as training proceeds, we know that the stochastic gradient would not be clipped after some iteration $k$ if we use the constant clipping bound $\bar{C}$ as Yu \textit{et al.}~\shortcite{yu2021decentralized}; Xu \textit{et al.}~\shortcite{xu2021dp}; Li and Chi~\shortcite{li2025convergence} did, thus resulting in adding unnecessary excessive noise calibrated with $\bar{C}$ in the later stage of training. To address this issue, we employ a dynamic gradient clipping scheme for each node to reduce the clipping bounds $C_k$ across the updates. Compared to the fixed clipping bound $\bar{C}$ used in~\cite{yu2021decentralized,xu2021dp,li2025convergence}, it can reduce the noise level after a particular $k$ when $C_k < \bar{C}$, which is beneficial for stabilizing the updates. In particular, we set the clipping bound $C_k$ as $C_k=C_0\cdot \left( \rho _c \right) ^{-\frac{k}{K}}$, where $\rho _c > 1$ is the hyper-parameter to control the decay rate of the clipping bound, $C_0$ is the initial clipping bound, and $K$ is the total number of iterations.

\textbf{Dynamic increasing per-step privacy budget}. Given a total privacy budget, the existing differentially private decentralized learning methods~\cite{yu2021decentralized,xu2021dp,li2025convergence} uniformly distribute privacy budgets across all training steps. However, recent works~\cite{zhu2019deep,wei2021gradient,wei2023securing} point out that it is relatively easier to leak privacy at the initial stage of the training process, and it becomes increasingly difficult as the training progresses.
To this end, we allocate a small privacy budget in the early stages and gradually increase privacy budgets, i.e., setting $\mu_{k+1} \geqslant \mu_{k}$ for all $k \in \{0,1,2,...,K-1\}$. In addition, according to~\eqref{def_of_mu}, we observe that small (resp. large) $\mu_k$ means adding large (resp. small) noise. Therefore, setting $\mu_{k+1} \geqslant \mu_{k}$ implies adding large (resp., small) noise in the early (resp., later) stages of decentralized training, which helps improve model accuracy. 
The intuition is that at the beginning of training, the model is far away from the optimum, and the gradient magnitudes are usually large (c.f., Figure~\ref{demo_gradient_convergence}); larger noise is thus allowed and even helps to quickly escape the saddle point~\cite{ge2015escaping}. As training proceeds, the model approaches the optimum and the gradient magnitude converge, smaller noise is then desired to stabilize the update for convergence.

The following proposition provides a way for privacy accounting of dynamic non-uniform $\mu_k$-GDP costs throughout the whole training process.

\begin{Pro}[Composition theorem for GDP with varying $\mu_k$~\cite{du2021dynamic}]
\label{composition}
Consider a series of random mechanisms $M_k$ for $k\in \left\{ 0,1,2,...,K-1 \right\}$, where $M_k$ is $\mu_k$-GDP, and each mechanism works only on a subsampled dataset by independent Bernoulli trial with probability $p$. 
After $K$ steps by composition of Gaussian mechanism, $M\triangleq M_{K-1}\otimes \cdot \cdot \cdot \otimes M_1\otimes M_0$ satisfies $\mu_{tot}$-GDP where
\begin{equation}
\label{mu_total}
\mu _{tot}=p\cdot \sqrt{\sum_{k=0}^{K-1}{\left( e^{\mu _{k}^{2}}-1 \right)}} .
\end{equation}
\end{Pro}


To this end, we set the per-step privacy budget $\mu_k$ as
\begin{equation}
\label{growing_mu}
\mu _k=\mu _0\cdot \left( \rho _{\mu} \right) ^{\frac{k}{K}} ,
\end{equation}
where $\rho _{\mu}>1$ is the hyper-parameter controlling the growth rate of $\mu_k$, and $\mu_0$ is the initial privacy budget.
Given the target total privacy budget $\left( \epsilon, \delta \right)$, the corresponding privacy budget in the GDP framework $\mu_{tot}$ can be obtained by~\eqref{privacy_tranfer}.
Substituting~\eqref{growing_mu} into~\eqref{mu_total} with $p=\frac{1}{J}$, we have
\begin{equation}
\label{solution_of_mu_0}
J^2\mu _{tot}^{2}=\sum_{k=0}^{K-1}{\left( e^{\left( \mu _0\cdot \left( \rho _{\mu} \right) ^{\frac{k}{K}} \right) ^2}-1 \right)},
\end{equation}
and $\mu_0$ can be computed using a numerical method such as binary search. 
With a specific value of $\mu_0$ and $\rho_\mu$, the value of $\mu_k$ at each iteration $k$ can be calculated by~\eqref{growing_mu}.

In addition, we provide two by-product differentially private decentralized learning algorithms (termed {\algS} and {\algmu}), which employ only dynamic clipping bound reduction method and dynamic per-step privacy budget growth method, respectively:
\begin{itemize}
    \item {\algS}: we set the clipping bound $C_k$ as $C_k=C_0\cdot \left( \rho _c \right) ^{-\frac{k}{K}}$ (c.f., Algorithm~\ref{Adpt_PrivSGP_combine}), while maintaining the per-step privacy budget $\mu_k$ the same for all iterations, i.e., fixing $\mu_k=\bar{\mu}$, $\forall k$. By substituting $\mu_k=\bar{\mu}$ into~\eqref{mu_total} in Proposition~\ref{composition} with $p=\frac{1}{J}$, we obtain the closed-form solution of $\bar{\mu}$:
    \begin{equation}
    \label{bar_mu}
    \bar{\mu}=\sqrt{\log \left( \frac{J^2\mu _{tot}^{2}}{K}+1 \right)} .
    \end{equation}
    With specific values of $C_k$ and $\bar{\mu}$, we can calibrate the DP noise variance at each iteration by $\sigma _k=\frac{C_k}{\bar{\mu}}=\frac{C_0}{\bar{\mu}}\cdot \left( \rho _c \right) ^{-\frac{k}{K}}$. The complete pseudocode can be found in Algorithm~\ref{Adpt_PrivSGP} in the Appendix~\ref{missed_pseudocode}.
    \item {\algmu}: we set the per-step privacy budget $\mu_k$ as $\mu _k=\mu _0\cdot \left( \rho _{\mu} \right) ^{\frac{k}{K}}$ (c.f., Algorithm~\ref{Adpt_PrivSGP_combine}), where $\mu_0$ is the solution of~\eqref{solution_of_mu_0}, while employing the fixed clipping bound, i.e., $C_k=\bar{C}$, $\forall k$. Therefore, the DP noise variance $\sigma_k$ at each iteration can be calculated by $\sigma _k=\frac{\bar{C}}{\mu _k}=\frac{\bar{C}}{\mu _0}\cdot \left( \rho _{\mu} \right) ^{-\frac{k}{K}}$. The complete pseudocode can be found in Algorithm~\ref{Adpt_PrivSGP_mu} in the Appendix~\ref{missed_pseudocode}.
\end{itemize}

It can be observed that, given $\rho_c$ and $\rho_\mu$, the noise decay rate of {\alg} is faster than that of {\algS} and {\algmu}.
We also present an algorithm termed {\algbase} as our baseline, which employs constant clipping bound (i.e., fixing $C_k=\bar{C}$, $\forall k$) and uniformly distributes privacy budgets across updates (i.e., $\mu_k=\bar{\mu}$, $\forall k$) as usually did in most of existing DP-based decentralized methods~\cite{yu2021decentralized,xu2021dp,li2025convergence}.
According to the values of $\bar{C}$ and $\bar{\mu}$ (c.f.,~\eqref{bar_mu}), 
the constant DP noise variance is $\bar{\sigma}=\frac{\bar{C}}{\bar{\mu}}$.
The pseudocode of {\algbase} is summarized in Algorithm~\ref{Const_PrivSGP} in the Appendix~\ref{missed_pseudocode}.

\begin{algorithm}[t!]
\caption{General Form of {\alg}}
\label{general_Adpt_PrivSGP_Combine} 
\begin{algorithmic}[1]
\STATE \textbf{Initialization:}  DP budget $(\epsilon,\delta)$, $x_{i}^{0}=z_{i}^{0}\in \mathbb{R}^d$, $w_i^0=1$, step size $\gamma > 0$, total number of iterations $K$, clipping bounds $C_0,...,C_{K-1}$, and noise levels $\tilde{\sigma}\cdot \tilde{\sigma}_0,....,\tilde{\sigma}\cdot \tilde{\sigma}_{K-1}$.

\FOR{$k=0,1,...,K-1$, at node $i$,}
\STATE Follows from line 3 and 4 in Algorithm~\ref{Adpt_PrivSGP_combine};
\STATE Gradient clipping: $g_{i}^{k}=\mathrm{Clip}\left( \nabla f_i\left( z_{i}^{k};\xi _{i}^{k} \right) ;C_k \right) $;
\STATE Draws randomized noise $N_i^k$ from the Gaussian distribution: $N_{i}^{k}\sim \mathcal{N}\left( 0,\tilde{\sigma}^2\mathbb{I}_d \right)$;
\STATE Differentially private local SGD:
\begin{equation*}
x_{i}^{k+\frac{1}{2}}=x_{i}^{k}-\gamma (g_{i}^{k}+\tilde{\sigma}_kN_{i}^{k});
\end{equation*}
\STATE Follows from line 11-14 in Algorithm~\ref{Adpt_PrivSGP_combine}.
\ENDFOR
\end{algorithmic} 
\end{algorithm}

\section{Theoretical Analysis}
In this section, we theoretically investigate the impact of our dynamic noise strategy on model utility guarantee and provide the utility bound for proposed {\alg} given a certain privacy budget. We first present a general form of {\alg} where $C_k$ and $\sigma_k$ can be arbitrarily predefined sequences (c.f., Algorithm~\ref{general_Adpt_PrivSGP_Combine}). 
Different from Algorithm~\ref{Adpt_PrivSGP_combine}, we denote $\sigma_k=\tilde{\sigma}\cdot \tilde{\sigma}_k$ for convenience in subsequent analysis, and we term $\tilde{\sigma}$ as noise scale without loss of generality.

Then, we show that the DP guarantee for each node in Algorithm~\ref{general_Adpt_PrivSGP_Combine} can be achieved by setting the noise scale $\tilde{\sigma}$ properly according to the given total privacy budget $(\epsilon,\delta)$ as well as the equivalent privacy parameter $\mu_{tot}$, which is summarized in the following proposition.

\begin{Pro}[Privacy guarantee]
\label{lemma_of_noise_scale}
Let $K$ be the total number of iterations. Algorithm~\ref{general_Adpt_PrivSGP_Combine} achieves $(\epsilon,\delta)$-DP guarantee for each node if we set the noise scale as
\begin{equation}
\label{value_of_noise_scale}
\tilde{\sigma} = \frac{1}{J\mu _{tot}}\sqrt{2\sum_{k=0}^{K-1}{\frac{C_{k}^{2}}{\tilde{\sigma}_{k}^{2}}}},
\end{equation}
where $\mu_{tot}$ is the solution of \eqref{privacy_tranfer} with $\mu=\mu_{tot}$.
\end{Pro}

\begin{proof}
The complete proof can be found in the Appendix~\ref{proof_of_pro}.
\end{proof}

Next, we make the following commonly used assumption for the utility analysis of Algorithm~\ref{general_Adpt_PrivSGP_Combine}.
\begin{Ass}[$L$-smoothness]
\label{assumption_smooth} 
For each function $f_i,i\in\mathcal{V}$, there exists a constant $L>0$ such that $\left\| \nabla f_i\left( x \right) -\nabla f_i\left( y \right) \right\| \leqslant L\left\| x-y \right\| 
$. 
\end{Ass}



Suppose Assumptions~\ref{assumption_mixing_matrix}-\ref{assumption_smooth} hold, and assume that each per-sample gradient is upper-bounded, i.e., $\left\| \nabla f_i\left( z;\xi _i \right) \right\| \leqslant \varLambda$. Then, we are ready to provide the following theorem to characterize the utility guarantee of Algorithm~\ref{general_Adpt_PrivSGP_Combine}.

\begin{Thm} [Model utility]
\label{main_Theorem}
If we set the noise scale $\tilde{\sigma}$ as in~\eqref{value_of_noise_scale}, Algorithm~\ref{general_Adpt_PrivSGP_Combine} can achieve $(\epsilon,\delta)$-DP guarantee for each node and has the following utility guarantee
\begin{equation}
\label{long_eq}
\begin{aligned}
& \frac{1}{K}\sum_{k=0}^{K-1}{\mathbb{E}\left[ \left\| \nabla f\left( \bar{x}^k \right) \right\| ^2 \right]}  
\\
\leqslant &  \frac{2\left( f\left( \bar{x}^0 \right) -f^* \right)}{\gamma K} +\frac{3L^2 \varPsi^2\sum_{i=1}^n{\left\| x_{i}^{0} \right\| ^2}}{\left( 1-q \right) ^2nK} 
\\
&+\frac{3\gamma ^2L^2 \varPsi^2}{\left( 1-q \right) ^2}\cdot \frac{1}{K}\sum_{k=0}^{K-1}{\cdot \frac{1}{n}\sum_{i=1}^n{\mathbb{E}\left[ \left\| g_{i}^{k} \right\| ^2 \right]}}
\\
&+\gamma L\cdot \frac{1}{K}\sum_{k=0}^{K-1}{\mathbb{E}\left[ \left\| \frac{1}{n}\sum_{i=1}^n{g_{i}^{k}} \right\| ^2 \right]}
\\
&+\left( \frac{6\gamma ^2KL^2\varPsi ^2d}{\left( 1-q \right) ^2J^2\mu _{tot}^{2}}+\frac{2\gamma KLd}{nJ^2\mu _{tot}^{2}} \right) \underset{\text{$\triangleq T_1$: privacy noise term}}{\underbrace{\frac{1}{K}\sum_{k=0}^{K-1}{\frac{C_{k}^{2}}{\tilde{\sigma}_{k}^{2}}} \frac{1}{K}\sum_{k=0}^{K-1}{\tilde{\sigma}_{k}^{2}}}}
,
\\
&+\underset{\text{$\triangleq T_2$: bias term}}{\underbrace{2\mathbb{E}\left[ \frac{1}{K}\sum_{k=0}^{K-1}{\varLambda \left\| \nabla f\left( \bar{x}^k \right) \right\| \cdot \frac{1}{n}\sum_{i=1}^n{\mathcal{P}_{i}^{k}\left( C_k \right)}} \right] }}
\end{aligned}
\end{equation}
where $\bar{x}^k=\frac{1}{n}\sum_{i=1}^n{x_{i}^{k}}$, $f^*=\underset{x\in \mathbb{R}^d}{\min}f\left( x \right)$ while $\varPsi$ and $q \in [0,1)$ are positive constants\footnote{$q$ characterizes the speed of information propagation over the network. A smaller value of $q$ indicates faster propagation.} depending on the diameter of the network $\bigtriangleup $ and the sequence of mixing matrices $P^k$, whose definition can be found in Lemma~\ref{def_of_C_and_q} in the Appendix, $\mathcal{P}_i^k \left ( C_k \right )$ denotes the probability of a stochastic gradient being clipped with clipping bound $C_k$ for node $i$ at iteration $k$.
\end{Thm}
\begin{proof}
The complete proof can be found in the Appendix~\ref{proof_of_main_theorem}.
\end{proof}

\begin{Rem}[On clipping bounds $C_k$]
\label{main_remark}
The clipping of gradients will destroy the unbiased estimate of the local full gradient for each node, resulting in a constant bias error (reflected by $T_2$ in~\eqref{long_eq}). It is obvious that the smaller the clipping bound $C_k$, the greater the probability $\mathcal{P}_i^k \left( C_k \right)$ that clipping occurs at iteration $k$. Based on this fact, we know that in general, the bias term $T_2$ in~\eqref{long_eq} will be small if $C_k$'s are large and vise versa, given the same distribution of gradients. This implies that we can use large clipping bounds to reduce the bias term $T_2$. However, this may make the privacy noise error term $T_1$ in~\eqref{long_eq} become larger. It is thus essential to choose a proper sequence of $C_k$ that could effectively balance $T_1$ and  $T_2$. However, finding an optimal sequence of clipping bounds to achieve the best trade-off necessitates knowledge of the distribution of the stochastic gradient, which is impossible to obtain in practice. The existing literature~\cite{andrew2021differentially} points out that a noteworthy practical way is to keep the probability $\mathcal{P}_i^k \left( C_i^k \right)$ approximately constant. Considering that the norm of the gradient is decreasing as training progresses, $C_k$ should also be decreasing so as to keep $\mathcal{P}_i^k \left( C_k \right)$ roughly constant, which supports the design of our algorithm (c.f., Algorithm~\ref{Adpt_PrivSGP_combine}).
\end{Rem}

Next, we specify the sequences $C_k$ and $\tilde{\sigma}_k$ in Algorithm~\ref{general_Adpt_PrivSGP_Combine}, yielding a special instance of dynamic clipping bound and noise scheduling as used in Algorithm~\ref{Adpt_PrivSGP_combine}. We then provide the corresponding utility bound in the following corollary.

\begin{Cor}
\label{Corollary_after_theorem}
Under the same conditions of Theorem~\ref{main_Theorem}, by setting $C_k=\Theta \left( \left( \rho _c \right) ^{-\frac{k}{K}} \right) $ and $\tilde{\sigma}_k=\Theta \left( \left( \rho _c\cdot \rho _{\mu} \right) ^{-\frac{k}{K}} \right) $ for some $\rho_c >1$ and $\rho_\mu >1$, if we set $\gamma =\frac{1}{\sqrt{n}J\mu _{tot}}$, $\gamma K=\sqrt{n}J\mu _{tot}$ and assume that $J\mu _{tot}>\sqrt{n}$, we have
\begin{equation}
\label{utility_bound}
\frac{1}{K}\sum_{k=0}^{K-1}{\mathbb{E}\left[ \left\| \nabla f\left( \bar{x}^k \right) \right\| ^2 \right]}\leqslant \mathcal{O}\left( \frac{1}{\left( 1-q \right) ^2\sqrt{n}J\mu _{tot}} \right) +T_2,
\end{equation}
where $\mathcal{O}(\cdot)$ hides some constants, e.g., $L$, $\sum_{i=1}^n{\left\| x_{i}^{0} \right\| ^2}$, $\varLambda$.
\end{Cor}

\begin{proof}
The complete proof can be found in the Appendix~\ref{proof_of_corollary}.
\end{proof}

\begin{Rem}[Tight utility bound]
\label{remark_tight_utility_bound}
Setting aside the non-vanishing bias term $T_2$, the utility bound derived in Corollary 1 exhibits explicit dependency on the network-related parameter $q$, and reveals a scaling factor of $1/\sqrt{n}$ with respect to the number of nodes $n$, which has never been observed in the previous differentially private decentralized algorithms~\cite{yu2021decentralized,xu2021dp}.
By setting $n=1$ which corresponds to the centralized/single-node case, our utility result~\eqref{utility_bound} can be reduced to $\mathcal{O}\left( 1/J\mu_{tot} \right)$, sharing the same polynomial order as existing analysis~\cite{zhang2017efficient,du2021dynamic}. 
To our best knowledge, we provide the first theoretical utility analysis for decentralized non-convex stochastic optimization with dynamic gradient clipping bounds and noise levels, highlighting the utility-privacy trade-off.
\end{Rem}

\begin{Rem}[Non-trivial analysis] 
The exponential decay scheme of the sequences $C_k$ and $\tilde{\sigma}_k$ aligns with that of the sequences $C_k$ and $\sigma_k$ in Algorithm~\ref{Adpt_PrivSGP_combine}, up to a constant factor. Therefore, the theoretical result (utility bound) in Corollary~\ref{Corollary_after_theorem} corroborates the parameter settings in Algorithm~\ref{Adpt_PrivSGP_combine}. We note that our analysis is non-trivial and can not be directly derived by extending the existing analysis~\cite{du2021dynamic} for the centralized/single-node case. The authors in~\cite{du2021dynamic} assume $C_k$ and $\tilde{\sigma}_k$ to be constant when deriving utility bounds, thereby simplifying the analysis (c.f., Theorem~1 in~\cite{du2021dynamic}).
In contrast, our analysis involves a more precise setting of $C_k$ and $\tilde{\sigma}_k$ as outlined in Corollary~\ref{Corollary_after_theorem}, which requires a more sophisticated analysis (c.f.,~\eqref{involve_1} to~\eqref{involve_2} in the Appendix for the derivation). We will also evaluate the impact of two hyper-parameters $\rho_\mu$ and $\rho_c$ on algorithm performance in the experimental section. The fixed schedule for the clipping bound and noise level in Corollary~\ref{Corollary_after_theorem} is not the only option for our general algorithm form (Algorithm~\ref{general_Adpt_PrivSGP_Combine}). 
Exploring alternative dynamic clipping bound mechanisms--such as using more complex schedules than exponential decaying method or adaptive clipping bound calculated based on the real-time gradient norm--would require additional heuristic efforts and complicate the theoretical utility analysis. We leave these explorations for future work.
\end{Rem}

\section{Experiments}
We conduct extensive experiments to verify the performance of proposed {\alg} (c.f., Algorithm~\ref{Adpt_PrivSGP_combine}), with comparison to our two by-product algorithms {\algS} and {\algmu}, and two baselines: i) {\algbase}, which employs fixed-level noise; ii) non-private decentralized learning algorithm, which does not use gradient clipping or add DP noise, and thus serves as the upper bound of model accuracy. All experiments are deployed in a server with Intel Xeon E5-2680 v4 CPU @ 2.40GHz and 8 Nvidia RTX 3090 GPUs, and are implemented with distributed communication package \textit{torch.distributed} in PyTorch~\cite{paszke2017pytorch}, where a process serves as a node, and inter-process communication is used to mimic communication among nodes.

\begin{table}[!t]
\centering
\begin{tabular}{|l|l|l|l|l|}
\hline
\rule{0pt}{9pt}
\textbf{Algorithm} & $\boldsymbol{\epsilon=0.3}$ & $ \boldsymbol{\epsilon=0.7}$ & $\boldsymbol{\epsilon=1}$ & $\boldsymbol{\epsilon=3}$ \\
\hline
\rule{0pt}{9pt}
Non-Private & \multicolumn{4}{c|}{82.16} \\
\hline
\rule{0pt}{9pt}
{\algbase} & 46.37 & 52.38 & 60.02 & 71.11 \\
\hline
\rule{0pt}{9pt}
{\algmu} & 64.12 & 66.21 & 69.37 & 76.3 \\
\hline
\rule{0pt}{9pt}
{\algS} & 66.93 & 68.37 & 72.02 & 77.65 \\
\hline
\rule{0pt}{9pt}
{\alg} & \textbf{71.87} & \textbf{72.14} & \textbf{75.06} & \textbf{79.89} \\
\hline
\end{tabular}
\caption{Final testing accuracy (\%) for different algorithms when training ResNet-18 on Cifar-10 dataset, under different values of privacy budget $\epsilon$.}
\label{table_cifar}
\end{table}

\begin{figure}[!t]
\centering
\subfigure[Training loss]{
\includegraphics[width=0.473\linewidth]{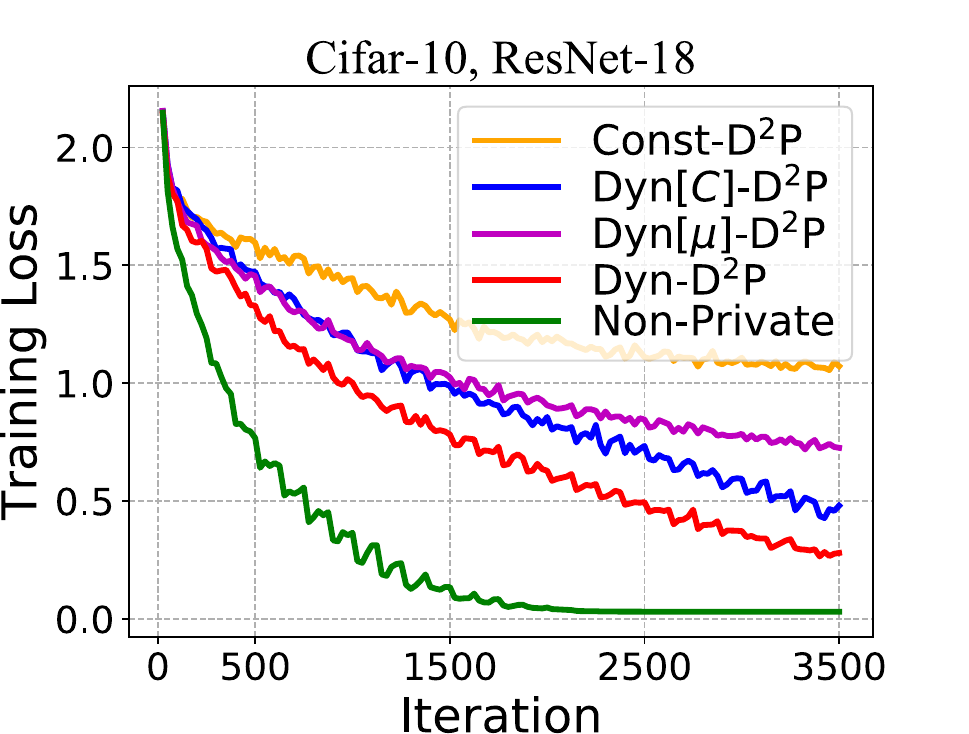}
}
\subfigure[Testing accuracy]{
\includegraphics[width=0.473\linewidth]{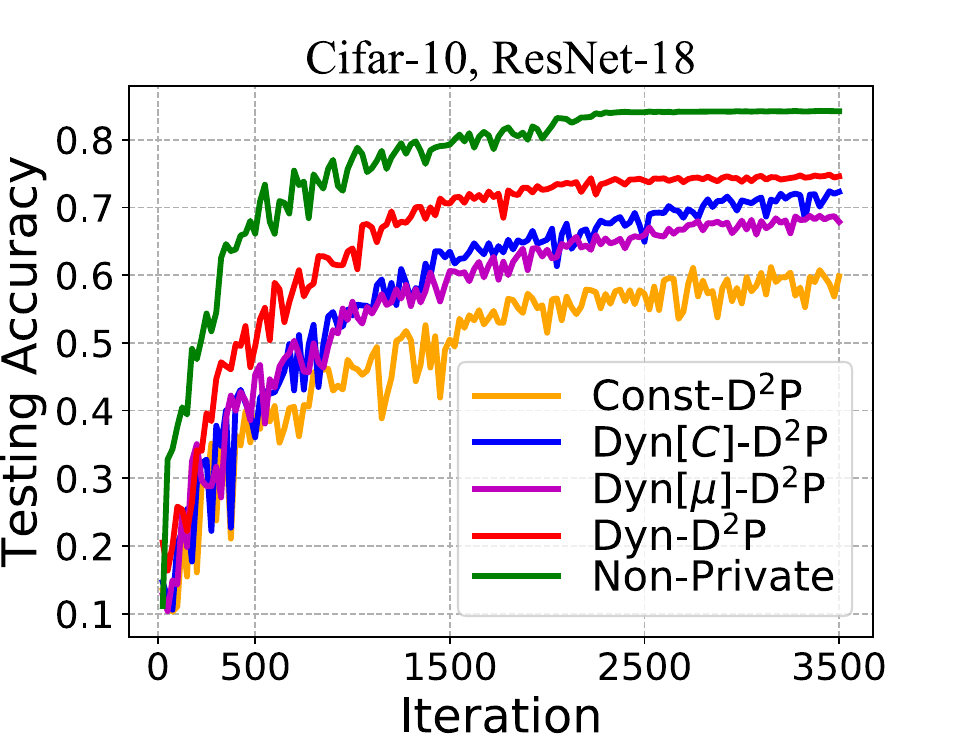}
}
\caption{Comparison of convergence performance for five algorithms under $(1, 10^{-4})$-DP guarantee for each node, when training ResNet-18 on Cifar-10 dataset.}
\label{illustration_cifar_iter}
\end{figure}


\subsection{Experimental Setup}
We compare five algorithms in a fully decentralized setting composed of 20 nodes, on two benchmark non-convex learning tasks: i) training ResNet-18~\cite{he2016deep} on Cifar-10~\cite{krizhevsky2009learning} dataset; ii) training shallow CNN model (composed of two convolution layers and two fully connected layers) on FashionMnist~\cite{xiao2017fashion} dataset.
We split shuffled datasets evenly to 20 nodes. For communication topology, unless otherwise stated, we use a time-varying directed exponential graph (refer to Appendix~\ref{missing_definition_of_graph} for its definition). The learning rate is set to be $0.05$ for ResNet-18 training and $0.03$ for shallow CNN model training. Privacy parameters $\delta$ is set to be $10^{-4}$, and we test different values for $\epsilon$ which implies different levels of privacy guarantee. Other parameters such as $\bar{C}$, $C_0$, $\rho_c$ and $\rho_\mu$ are detailed in the Appendix~\ref{experimental_set_up_parameter}. Note that all experimental results are averaged over five repeated runs.

\begin{table}[!t]
\centering
\begin{tabular}{|l|l|l|l|l|}
\hline
\rule{0pt}{9pt}
\textbf{Algorithm} & $\boldsymbol{\epsilon=0.3}$ & $ \boldsymbol{\epsilon=0.7}$ & $\boldsymbol{\epsilon=1}$ & $\boldsymbol{\epsilon=3}$ \\
\hline
\rule{0pt}{9pt}
Non-Private & \multicolumn{4}{c|}{89.98} \\
\hline
\rule{0pt}{9pt}
{\algbase} & 45.37 & 58.63 & 74.65 & 80.81 \\
\hline
\rule{0pt}{9pt}
{\algmu} & 81.12 & 82.98  & 82.23 & 84.36 \\
\hline
\rule{0pt}{9pt}
{\algS} & 82.93 & 83.65  & 84.06 & 84.98 \\
\hline
\rule{0pt}{9pt}
{\alg} & \textbf{84.88} & \textbf{85.36}  & \textbf{86.21} & \textbf{87.89} \\
\hline
\end{tabular}
\caption{Final testing accuracy (\%) for different algorithms when training shallow CNN on FashionMnist dataset, under different values of privacy budget $\epsilon$.}
\label{table_fashion_mnist}
\end{table}

\begin{figure}[!t]
  \centering
  \subfigure[Training loss]{
    \includegraphics[width=0.473\linewidth]{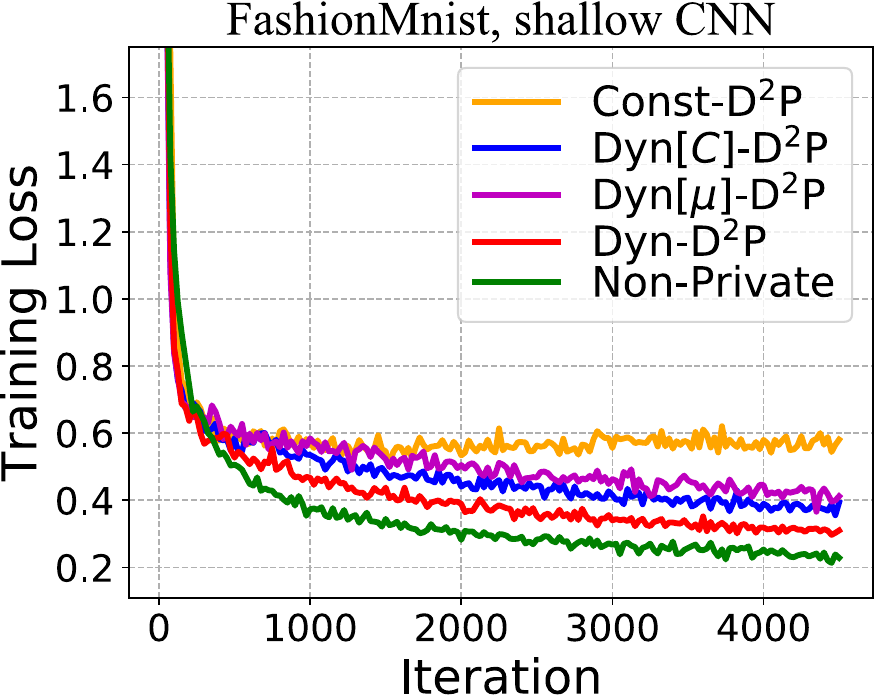}
  }
  \subfigure[Testing accuracy]{
    \includegraphics[width=0.473\linewidth]{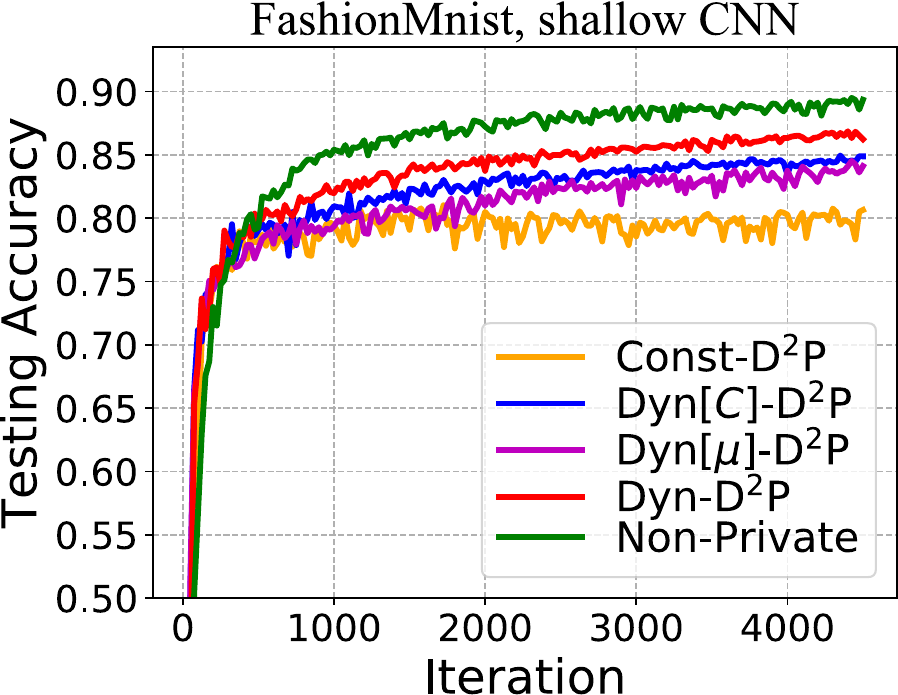}
  }
  \caption{Comparison of convergence performance for five algorithms under $(3, 10^{-4})$-DP guarantee for each node, when training shallow CNN on FashionMnist dataset.}
\label{illustration_fashion_mnist}
\end{figure}

\subsection{Superior Performance against Baseline Methods}

For the ResNet-18 training task, we present the experimental result of final model accuracy in Table~\ref{table_cifar}, and provide the plots of training loss/testing accuracy versus iteration in Figure~\ref{illustration_cifar_iter} with a privacy budget of $\epsilon = 1$ (the plots with other values of $\epsilon$ can be found in the Appendix~\ref{missed_plots}). 
It can be observed that our {\alg} and two by-product algorithms ({\algS} and {\algmu}) consistently outperform the baseline algorithm {\algbase} which employs constant noise. Among these above DP algorithms, {\alg} achieves the highest model accuracy while maintaining the same level of privacy protection. Furthermore, a comparison of experimental results with different privacy budgets (i.e., varying values of $\epsilon$) shows that the stronger the level of required privacy protection (i.e., the smaller the value of budget $\epsilon$), the more pronounced the advantage in model accuracy with our dynamic noise strategy. In particular, when setting a small $\epsilon=0.3$ which implies a strong privacy guarantee, our {\alg} achieves a $25\%$ higher model accuracy than {\algbase} employing constant noise strategies. These results verify the superiority of our dynamic noise approach.

For the shallow CNN training task, we present the experimental result of final model accuracy in Table~\ref{table_fashion_mnist}, and provide the plots of training loss/testing accuracy in terms of iteration in Figure~\ref{illustration_fashion_mnist} with a privacy budget of $\epsilon = 1$ (the plots with other values of $\epsilon$ can be found in the Appendix~\ref{missed_plots}). The takeaways from the experimental results are similar to previous experiments on the ResNet-18 training task, and our proposed {\alg} performs much better than the baseline algorithm {\algbase} while maintaining the same level of privacy protection, which again highlights the superiority of our dynamic noise approach. In particular, under a strong level of required privacy guarantee with budget $\epsilon=0.3$, our {\alg} achieves a $39\%$ higher model accuracy compared to {\algbase}.

\subsection{Sensitivity to Hyper-parameters}
In this part, we test the robustness of our algorithm (c.f., Algorithm~\ref{Adpt_PrivSGP_combine}) to different values of hyper-parameters $\rho_c$ and $\rho_\mu$. We use grid search to demonstrate the impact of these two parameters on the final model accuracy. It follows from the results as shown in Figure~\ref{robust_to_hyperparameters} that, our algorithm can almost maintain the final model accuracy and consistently improve model accuracy compared to {\algbase} across a wide range of $\rho_c$ and $\rho_\mu$, which implies that our proposed algorithm is robust to the value of hyper-parameter $\rho_c$ and $\rho_\mu$.

\begin{figure}[!h]
\centering
\subfigure{
\includegraphics[width=0.474\linewidth]{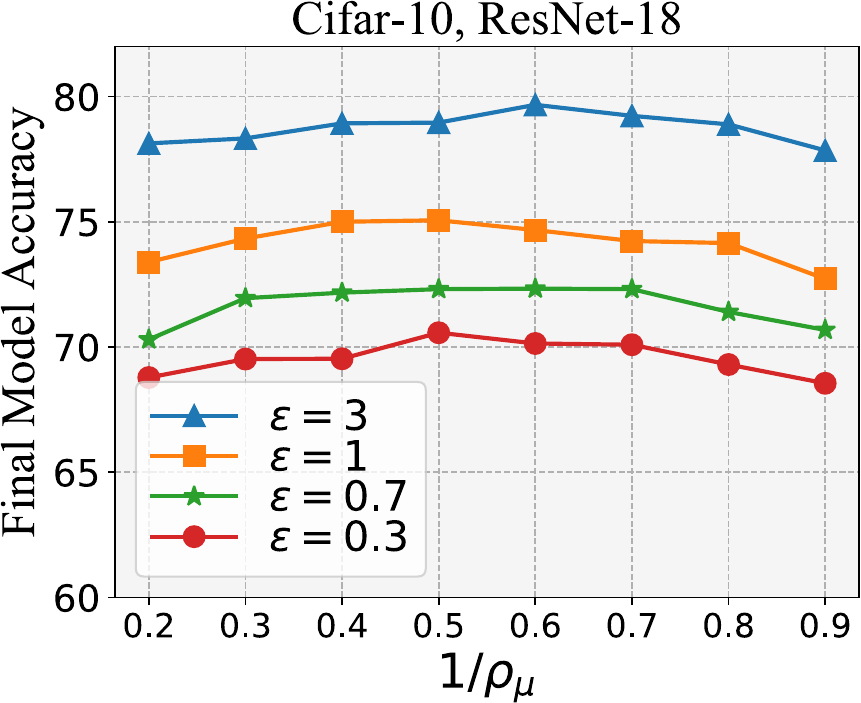}
}
\subfigure{
\includegraphics[width=0.474\linewidth]{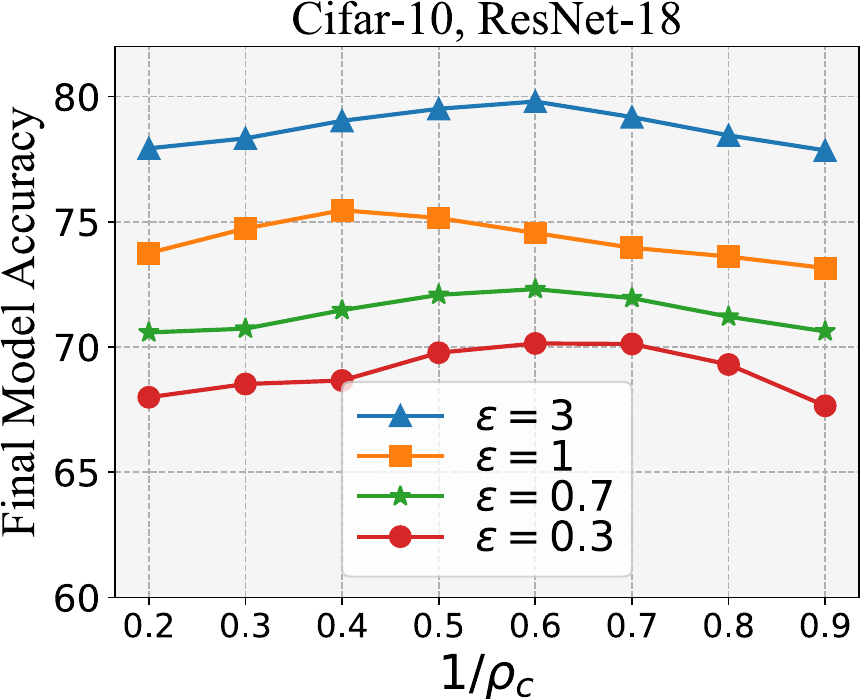}
}
\caption{Illustration of the robustness of {\alg} against the values of $\rho_\mu$ and $\rho_c$ under different privacy budget $\epsilon$.}
\label{robust_to_hyperparameters}
\end{figure}

\begin{figure}[!h]
\centering
\subfigure[$(1, 10^{-4})$-DP guarantee]{
\includegraphics[width=0.474\linewidth]{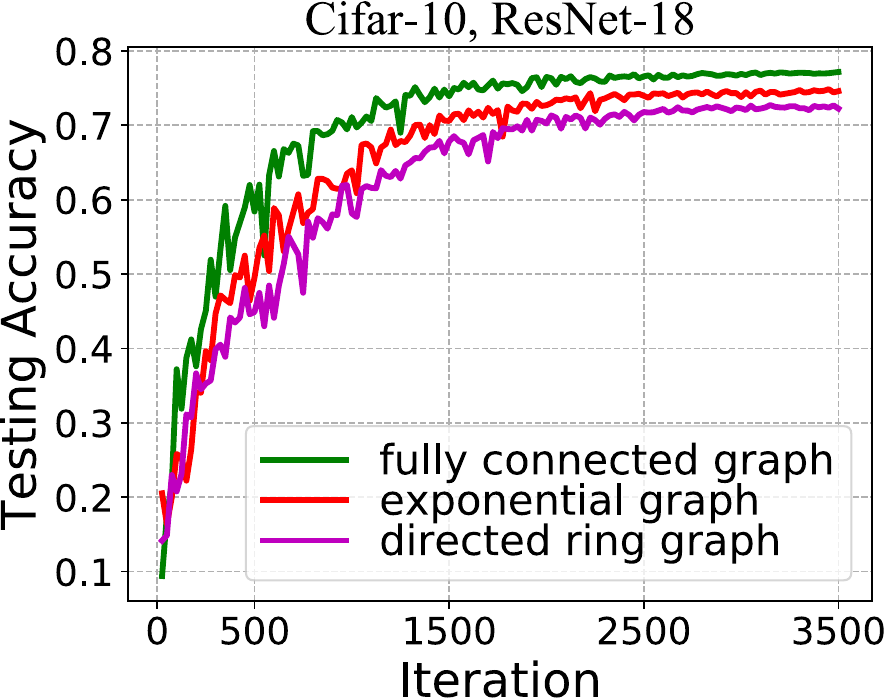}
\label{different_comm_graphs}
}
\subfigure[$(1, 10^{-4})$-DP guarantee ]{
\includegraphics[width=0.4755\linewidth]{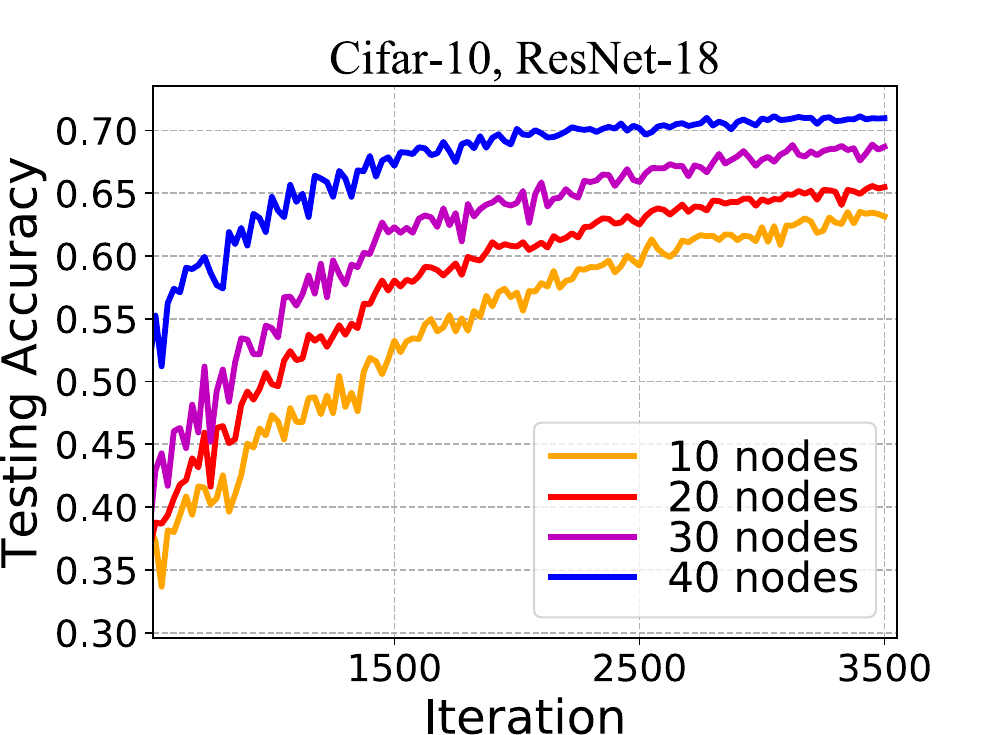}
\label{different_node_numbers}
}
\caption{Comparison of convergence performance (model utility) for {\alg} over (a) different graphs consisting of 20 nodes and (b) an exponential graph with different number of nodes, respectively.}
\end{figure}

\subsection{Performance over Different Graphs and Node Numbers}
First, we implement {\alg} over a static directed ring graph, a time-varying directed exponential graph, and a fully connected graph, where the $q$ value w.r.t. these three graphs decreases in sequence. The experimental result shown in Figure~\ref{different_comm_graphs} demonstrates that the model accuracy (utility) increases in sequence across these graphs. Then, we implement {\alg} on a time-varying directed exponential graph consisting of different numbers of nodes (refer to Appendix~\ref{missing_experimental_set_up_node_number} for more details on the setup). It can be observed from Figure~\ref{different_node_numbers} that, increasing the number of nodes improves the model accuracy (utility). Note that both of these above observations align with the theoretical insights outlined in Corollary~\ref{Corollary_after_theorem} and Remark~\ref{remark_tight_utility_bound}.

\section{Conclusion}
In this work, we proposed a differentially private decentralized learning method {\alg}, for non-convex optimization problems, which dynamically adjusts gradient clipping bounds and noise levels across the update. The proposed dynamic noise strategy allows us to enhance the model accuracy while maintaining the level of privacy guarantee. Extensive experiments show that our {\alg} outperforms the existing counterparts with fixed-level noises, especially under strong privacy levels. Our analysis shows that the utility bound of Dyn-D$^2$P exhibits an explicit dependency on the network-related parameter and enjoys a scaling factor of $1/\sqrt{n}$, up to a bias error term induced by gradient clipping. To our knowledge, we provide the first theoretical utility analysis for fully decentralized non-convex stochastic optimization with dynamic gradient clipping bounds and noise levels, highlighting the utility-privacy trade-off. We will focus on eliminating the bias term induced by gradient clipping in the future.

\section*{Acknowledgments}
This work is supported in parts by the National
Key R\&D Program of China under Grant No. 2022YFB3102100, and in parts by National Natural Science Foundation of China under Grants 62373323, 62088101, 62402256.

\bibliographystyle{named}
\bibliography{ijcai25}

\onecolumn
\appendix
\begin{center}
\LARGE{\textbf{Appendix}}
\end{center}
{\footnotesize
\tableofcontents
}

\section{Proof of Theorem~\ref{main_Theorem}}
\label{proof_of_main_theorem}

To facilitate our analysis, we rewrite the update of Algorithm~\ref{general_Adpt_PrivSGP_Combine} in an equivalent compact form as follows:
\begin{equation}
\label{compact_form_iterate}
X^{k+1}=\left( X^k-\gamma \left( G^k+\tilde{\sigma}_kN^k \right) \right) \left( P^k \right) ^{\top}
\end{equation}
where $\left( P^k \right) ^{\top}\in \mathbb{R}^{n\times n}
$ is the transpose of the mixing matrix $P^k$ at iteration $k$, and

$X^{k}:=\left[ x_{1}^{k},x_{2}^{k},\cdot \cdot \cdot ,x_{n}^{k} \right] \in \mathbb{R}^{d\times n}
$: concatenation of all the nodes' parameters at iteration $k$;

$Z^{k}:=\left[ z_{1}^{k},z_{2}^{k},\cdot \cdot \cdot ,z_{n}^{k} \right] \in \mathbb{R}^{d\times n}$: concatenation of all the nodes' de-biased parameters at iteration $k$;

$\nabla F(Z^k;\xi ^k):=\left[ \nabla f_1(z_{1}^{k};\xi _{1}^{k}),\nabla f_2(z_{2}^{k};\xi _{2}^{k}),...,\nabla f_n(z_{n}^{k};\xi _{n}^{k}) \right] \in \mathbb{R}^{d\times n}
$: concatenation of all the nodes' stochastic gradients at iteration $k$;

$G^k:=\left[ g_{1}^{k},g_{2}^{k},...,g_{n}^{k} \right] \in \mathbb{R}^{d\times n}$: concatenation of all the nodes' clipped stochastic gradients at iteration $k$;

$N^k:=\left[ N_1^k,N_2^k,...,N_n^k \right] \in \mathbb{R}^{d\times n} $: concatenation of all the nodes' generated Gaussian noise at iteration $k$.

Now, let $\bar{x}^k=\frac{1}{n}X^k\mathbf{1}=\frac{1}{n}\sum_{i=1}^n{x_{i}^{k}}\in \mathbb{R}^d
$ denote the average of all nodes' parameters at iteration $k$. Then, the update of the average system of~\eqref{compact_form_iterate} becomes 
\begin{equation}
\label{average_system}
\bar{x}^{k+1}=\bar{x}^k-\gamma \cdot \left( \frac{1}{n}\sum_{i=1}^n{g_{i}^{k}}+\tilde{\sigma}_k\cdot \frac{1}{n}\sum_{i=1}^n{N_{i}^{k}} \right) ,
\end{equation}
which can be easily obtained by right multiplying $\frac{1}{n}\mathbf{1}$ from both sides of~\eqref{compact_form_iterate}
and using the column-stochastic property of $P^k$ (c.f., Assumption~\ref{assumption_mixing_matrix}). The above average system will be useful in subsequent analysis. In addition, we denote $\mathcal{F}^k:=\left\{ \bigcup\nolimits_{i=1}^n{\left( x_{i}^{0},z_{i}^{0},\xi _{i}^{0},N_{i}^{0},\cdot \cdot \cdot ,x_{i}^{k-1},z_{i}^{k-1},\xi _{i}^{k-1},N_{i}^{k-1},x_{i}^{k},z_{i}^{k} \right)} \right\} 
$ as the filtration of the history sequence up to iteration $k$, and define $\mathbb{E}_k\left[ \cdot \right] =\mathbb{E}\left[ \cdot \left| \mathcal{F}^k \right. \right] $ the conditional expectation given $\mathcal{F}^k$.

\subsection{Bounding accumulative consensus error}
\label{appendix_a}
We first provide two supporting lemmas to facilitate the subsequent analysis.

\begin{Lem}
\label{tecnical_lemma_1}
Let $\left\{ v^k \right\} _{k=0}^{\infty}$ be a non-negative sequence and $\lambda \in (0,1)$. Then, we have
\begin{equation}
    \left( \sum_{l=0}^k{\lambda ^{k-l}v^l} \right) ^2\leqslant \frac{1}{1-\lambda}\sum_{l=0}^k{\lambda ^{k-l}\left( v^l \right) ^2}.
\end{equation}
\end{Lem}

\begin{proof}
Using Cauchy-Swarchz inequality, we have
\begin{equation*}
\begin{aligned}
\left( \sum_{l=0}^k{\lambda ^{k-l}v^l} \right) ^2&=\left( \sum_{l=0}^k{\lambda ^{\frac{k-l}{2}}\left( \lambda ^{\frac{k-l}{2}}v^l \right)} \right) ^2
\\
&\leqslant \sum_{l=0}^k{\left( \lambda ^{\frac{k-l}{2}} \right) ^2}\cdot \sum_{l=0}^k{\left( \lambda ^{\frac{k-l}{2}}v^l \right) ^2}
\\
&\leqslant \frac{1}{1-\lambda}\sum_{l=0}^k{\lambda ^{k-l}\left( v^l \right) ^2},
\end{aligned}
\end{equation*}
which completes the proof.
\end{proof}

The following lemma bounds the distance between the de-biased parameters $z_i^k$ at each node $i$ and the node-wise average $\bar{x}^k$, which can be adapted from Lemma 3 in~\cite{assran2019stochastic}.

\begin{Lem}
\label{def_of_C_and_q}
Suppose that Assumptions~\ref{assumption_mixing_matrix} and~\ref{assumption_graph} hold.
Let $\varepsilon$ be the minimum of all non-zero mixing weights, $\lambda =1-n\varepsilon^{\bigtriangleup B}$ and $q=\lambda ^{\frac{1}{\bigtriangleup B+1}}
$. Then, there exists a constant 
\begin{equation}
\varPsi<\frac{2\sqrt{d}\varepsilon^{-\bigtriangleup B}}{\lambda ^{\frac{\bigtriangleup B+2}{\bigtriangleup B+1}}}
,
\end{equation}
such that for any $i \in \mathcal{V}$ and $k \geqslant 0$, we have
\begin{equation}
\label{consensus_error_initial}
\left\| z_{i}^{k}-\bar{x}^k \right\| \leqslant \varPsi q^k\left\| x_{i}^{0} \right\| +\gamma \varPsi \sum_{s=0}^k{q^{k-s}\left\| g_{i}^{s}+\tilde{\sigma}_sN_{i}^{s} \right\|}.
\end{equation}
\end{Lem}

Now, we attempt to upper bound the accumulative consensus error $\sum_{k=0}^{K-1}{\frac{1}{n}\sum_{i=1}^n{\mathbb{E}\left[ \left\| z_{i}^{k}-\bar{x}^k \right\| ^2 \right]}}$, which is summarized in the following lemma.

\begin{Lem}
\label{Lemma_consensus_error}
Suppose Assumptions~\ref{assumption_mixing_matrix} and~\ref{assumption_graph} hold. Then, we have
\begin{equation}
\label{upper_bound_sum_of_M_k}
\begin{aligned}
& \sum_{k=0}^{K-1}{\frac{1}{n}\sum_{i=1}^n{\mathbb{E}\left[ \left\| z_{i}^{k}-\bar{x}^k \right\| ^2 \right]}}
\\
\leqslant & \frac{3 \varPsi^2\sum_{i=1}^n{\left\| x_{i}^{0} \right\| ^2}}{\left( 1-q \right) ^2n}+\frac{3\gamma ^2 \varPsi^2}{\left( 1-q \right) ^2}\sum_{k=0}^{K-1}{\cdot \frac{1}{n}\sum_{i=1}^n{\mathbb{E}\left[ \left\| g_{i}^{k} \right\| ^2 \right]}} +\frac{3\gamma ^2 \varPsi^2d}{\left( 1-q \right) ^2}\sum_{k=0}^{K-1}{\tilde{\sigma}^2\tilde{\sigma}_{k}^{2}} ,
\end{aligned}
\end{equation}
\end{Lem}

\begin{proof}
According to~\eqref{consensus_error_initial}, we have
\begin{equation}
\label{before_squaring}
\begin{aligned}
\left\| z_{i}^{k}-\bar{x}^k \right\| \leqslant & \varPsi q^k\left\| x_{i}^{0} \right\| +\gamma \varPsi \sum_{s=0}^k{q^{k-s}\left\| g_{i}^{s} \right\|}+\gamma \varPsi \sum_{s=0}^k{q^{k-s}\left\| \tilde{\sigma}_sN_{i}^{s} \right\|}
\end{aligned}
\end{equation}

By squaring both sides of~\eqref{before_squaring}, we obtain
\begin{equation}
\label{after_squaring}
\begin{aligned}
& \left\| z_{i}^{k}-\bar{x}^k \right\| ^2
\\
\leqslant & 3 \varPsi^2q^{2k}\left\| x_{i}^{0} \right\| ^2+3\gamma ^2 \varPsi^2\left( \sum_{s=0}^k{q^{k-s}\left\| g_{i}^{s} \right\|} \right) ^2 +3\gamma ^2 \varPsi^2\left( \sum_{s=0}^k{q^{k-s}\left\| \tilde{\sigma}_sN_{i}^{s} \right\|} \right) ^2
\\
\leqslant & 3\varPsi^2q^{2k}\left\| x_{i}^{0} \right\| ^2+\frac{3\gamma ^2\varPsi^2}{1-q}\sum_{s=0}^k{q^{k-s}\left\| g_{i}^{s} \right\| ^2} +\frac{3\gamma ^2 \varPsi^2}{1-q}\sum_{s=0}^k{q^{k-s}\left\| \tilde{\sigma}_sN_{i}^{s} \right\| ^2},
\end{aligned}
\end{equation}
where we used $\left\| a+b+c \right\| ^2\leqslant 3\left\| a \right\| ^2+3\left\| b \right\| ^2+3\left\| c \right\| ^2$ in the first inequality and Lemma~\ref{tecnical_lemma_1} in the second, respectively.

Taking total expectation on both sides of~\eqref{after_squaring} yields
\begin{equation}
\label{each_node_bias}
\begin{aligned}
&\mathbb{E}\left[ \left\| z_{i}^{k}-\bar{x}^k \right\| ^2 \right] 
\\
\leqslant& 3 \varPsi^2q^{2k}\left\| x_{i}^{0} \right\| ^2+\frac{3\gamma ^2 \varPsi^2}{1-q}\sum_{s=0}^k{q^{k-s}\mathbb{E}\left[ \left\| g_{i}^{s} \right\| ^2 \right]}+\frac{3\gamma ^2 \varPsi^2}{1-q}\sum_{s=0}^k{q^{k-s}\mathbb{E}\left[ \left\| \tilde{\sigma}_sN_{i}^{s} \right\| ^2 \right]}
\\
=&3 \varPsi^2q^{2k}\left\| x_{i}^{0} \right\| ^2+\frac{3\gamma ^2 \varPsi^2}{1-q}\sum_{s=0}^k{q^{k-s}\mathbb{E}\left[ \left\| g_{i}^{s} \right\| ^2 \right]} +\frac{3\gamma ^2 \varPsi^2d}{1-q}\sum_{s=0}^k{q^{k-s}\cdot \tilde{\sigma}^2\tilde{\sigma}_{s}^{2}}.
\end{aligned}
\end{equation}

Summing~\eqref{each_node_bias} from $i=1$ to $n$ and dividing by $n$, we have
\begin{equation}
\label{in_eq_1}
\begin{aligned}
& \frac{1}{n}\sum_{i=1}^n{\mathbb{E}\left[ \left\| z_{i}^{k}-\bar{x}^k \right\| ^2 \right]}
\\
\leqslant & \frac{3 \varPsi^2q^{2k}}{n}\sum_{i=1}^n{\left\| x_{i}^{0} \right\| ^2} +\frac{3\gamma ^2 \varPsi^2}{1-q}\sum_{s=0}^k{q^{k-s}\cdot \frac{1}{n}\sum_{i=1}^n{\mathbb{E}\left[ \left\| g_{i}^{s} \right\| ^2 \right]}} +\frac{3\gamma ^2 \varPsi^2d}{1-q}\sum_{s=0}^k{q^{k-s}\cdot \tilde{\sigma}^2\tilde{\sigma}_{s}^{2}}.
\end{aligned}
\end{equation}
Further, summing~\eqref{in_eq_1} from $k=0$ to $K-1$ leads to
\begin{equation}
\begin{aligned}
& \sum_{k=0}^{K-1}{\frac{1}{n}\sum_{i=1}^n{\mathbb{E}\left[ \left\| z_{i}^{k}-\bar{x}^k \right\| ^2 \right]}}
\\
\leqslant & \frac{3 \varPsi^2\sum_{i=1}^n{\left\| x_{i}^{0} \right\| ^2}}{n}\cdot \sum_{k=0}^{K-1}{q^{2k}}+\frac{3\gamma ^2 \varPsi^2d}{1-q}\sum_{k=0}^{K-1}{\sum_{s=0}^k{q^{k-s}\cdot \tilde{\sigma}^2\tilde{\sigma}_{s}^{2}}}
\\
&+\frac{3\gamma ^2 \varPsi^2}{1-q}\sum_{k=0}^{K-1}{\sum_{s=0}^k{q^{k-s}\cdot \frac{1}{n}\sum_{i=1}^n{\mathbb{E}\left[ \left\| g_{i}^{s} \right\| ^2 \right]}}}
\\
\leqslant & \frac{3 \varPsi^2\sum_{i=1}^n{\left\| x_{i}^{0} \right\| ^2}}{\left( 1-q \right) ^2n}+\frac{3\gamma ^2 \varPsi^2}{\left( 1-q \right) ^2}\sum_{k=0}^{K-1}{\cdot \frac{1}{n}\sum_{i=1}^n{\mathbb{E}\left[ \left\| g_{i}^{k} \right\| ^2 \right]}} +\frac{3\gamma ^2 \varPsi^2d}{\left( 1-q \right) ^2}\sum_{k=0}^{K-1}{\tilde{\sigma}^2\tilde{\sigma}_{k}^{2}} ,
\end{aligned}
\end{equation}
where in the last inequality we used the fact that $\sum_{k=0}^{K-1}{\sum_{s=0}^k{q^{k-s}a_s}}\leqslant \frac{1}{1-q}\sum_{k=0}^{K-1}{a_k}$.
\end{proof}

\subsection{Proof of main result in Theorem~\ref{main_Theorem}}
Applying the descent lemma to $f$ at $\bar{x}^k$ and $\bar{x}^{k+1}$, we have
\begin{equation}
\begin{aligned}
f\left( \bar{x}^{k+1} \right) \leqslant & f\left( \bar{x}^k \right) +\left< \nabla f\left( \bar{x}^k \right) ,\bar{x}^{k+1}-\bar{x}^k \right> +\frac{L}{2}\left\| \bar{x}^{k+1}-\bar{x}^k \right\| ^2
\\
\overset{\eqref{average_system}}{=}&f\left( \bar{x}^k \right) -\gamma \left< \nabla f\left( \bar{x}^k \right) ,\frac{1}{n}\sum_{i=1}^n{g_{i}^{k}}+\tilde{\sigma}_k\cdot \frac{1}{n}\sum_{i=1}^n{N_{i}^{k}} \right> 
\\
&+\frac{\gamma ^2L}{2}\left\| \frac{1}{n}\sum_{i=1}^n{g_{i}^{k}}+\tilde{\sigma}_k\cdot \frac{1}{n}\sum_{i=1}^n{N_{i}^{k}} \right\| ^2
\\
=&f\left( \bar{x}^k \right) -\gamma \left< \nabla f\left( \bar{x}^k \right) ,\frac{1}{n}\sum_{i=1}^n{g_{i}^{k}}+\tilde{\sigma}_k\cdot \frac{1}{n}\sum_{i=1}^n{N_{i}^{k}} \right> +\frac{\gamma ^2L}{2}\left\| \frac{1}{n}\sum_{i=1}^n{g_{i}^{k}} \right\| ^2
\\
& +\frac{\gamma ^2L}{2}\tilde{\sigma}_{k}^{2}\left\| \frac{1}{n}\sum_{i=1}^n{N_{i}^{k}} \right\| ^2 +\gamma ^2L\left< \frac{1}{n}\sum_{i=1}^n{g_{i}^{k}},\tilde{\sigma}_k\cdot \frac{1}{n}\sum_{i=1}^n{N_{i}^{k}} \right> . 
\end{aligned}
\end{equation}

Taking the expectation of both sides conditioned on $\mathcal{F}^k$ for the above inequality, we obtain
\begin{equation}
\label{eq_1}
\begin{aligned}
\mathbb{E}_k\left[ f\left( \bar{x}^{k+1} \right) \right] \leqslant & f\left( \bar{x}^k \right) -\gamma \left< \nabla f\left( \bar{x}^k \right) ,\frac{1}{n}\sum_{i=1}^n{\mathbb{E}_k\left[ g_{i}^{k} \right]} \right> 
\\
& +\frac{\gamma ^2L}{2}\mathbb{E}_k\left[ \left\| \frac{1}{n}\sum_{i=1}^n{g_{i}^{k}} \right\| ^2 \right] +\frac{\gamma ^2L}{2}\tilde{\sigma}_{k}^{2}\mathbb{E}_k\left[ \left\| \frac{1}{n}\sum_{i=1}^n{N_{i}^{k}} \right\| ^2 \right] 
\\
= & f\left( \bar{x}^k \right) -\gamma \left< \nabla f\left( \bar{x}^k \right) ,\frac{1}{n}\sum_{i=1}^n{\mathbb{E}_k\left[ g_{i}^{k} \right]} \right> 
\\
& +\frac{\gamma ^2L}{2}\mathbb{E}_k\left[ \left\| \frac{1}{n}\sum_{i=1}^n{g_{i}^{k}} \right\| ^2 \right] +\frac{\gamma ^2Ld}{2n}\tilde{\sigma}^2\tilde{\sigma}_{k}^{2} .
\end{aligned}
\end{equation}

For the term of $\mathbb{E}_k\left[ g_{i}^{k} \right]$, we have
\begin{equation}
\label{eq_2}
\begin{aligned}
& \mathbb{E}_k\left[ g_{i}^{k} \right] 
\\
=&\mathbb{E}_{\xi _{i}^{k}}\left[ \mathrm{Clip}\left( \nabla f_i\left( z_{i}^{k};\xi _{i}^{k} \right) ;C_k \right) \right] 
\\
\overset{\left( i \right)}{=}& \sum_{j=1}^J{\frac{1}{J}\mathrm{Clip}\left( \nabla f_i\left( z_{i}^{k};j \right) ;C_k \right)}
\\
\overset{\left( ii \right)}{=}& \sum_{j=1}^J{\frac{1}{J}\mathbb{I}\left[ \left\| \nabla f_i\left( z_{i}^{k};j \right) \right\| \leqslant C_k \right] \cdot \nabla f_i\left( z_{i}^{k};j \right)}+\sum_{j=1}^J{\frac{1}{J}\mathbb{I}\left[ \left\| \nabla f_i\left( z_{i}^{k};j \right) \right\| >C_k \right] \cdot \frac{C_k\nabla f_i\left( z_{i}^{k};j \right)}{\left\| \nabla f_i\left( z_{i}^{k};j \right) \right\|}}
\\
\overset{\left( iii \right)}{=}& \frac{1}{J}\sum_{j=1}^J{\nabla f_i\left( z_{i}^{k};j \right)}+\sum_{j=1}^J{\left\{ \frac{1}{J}\mathbb{I}\left[ \left\| \nabla f_i\left( z_{i}^{k};j \right) \right\| >C_k \right] \cdot \nabla f_i\left( z_{i}^{k};j \right) \cdot \left( \frac{C_k}{\left\| \nabla f_i\left( z_{i}^{k};j \right) \right\|}-1 \right) \right\}}
\\
=& \nabla f_i\left( z_{i}^{k} \right) +\sum_{j=1}^J{\left\{ \frac{1}{J}\mathbb{I}\left[ \left\| \nabla f_i\left( z_{i}^{k};j \right) \right\| >C_k \right] \cdot \nabla f_i\left( z_{i}^{k};j \right) \cdot \left( \frac{C_k}{\left\| \nabla f_i\left( z_{i}^{k};j \right) \right\|}-1 \right) \right\}}.
\end{aligned}
\end{equation}
where $\mathbb{I}\left[ \cdot \right]$ is the indicator function; $(i)$ is based on the fact that each node uniformly at random chooses one out of $J$ data samples; $(ii)$ is due to the definition of gradient clipping operation;  $(iii)$ is according to the fact of $\mathbb{I}\left[ \left\| \nabla f_i\left( z_{i}^{k};j \right) \right\| \leqslant C_k \right] +\mathbb{I}\left[ \left\| \nabla f_i\left( z_{i}^{k};j \right) \right\| >C_k \right] =1$.


Now, substituting~\eqref{eq_2} into~\eqref{eq_1}, we have
\begin{equation}
\label{eq_3}
\begin{aligned}
\mathbb{E}_k\left[ f\left( \bar{x}^{k+1} \right) \right] \leqslant & f\left( \bar{x}^k \right) +A_1+A_2
+\frac{\gamma ^2L}{2}\mathbb{E}_k\left[ \left\| \frac{1}{n}\sum_{i=1}^n{g_{i}^{k}} \right\| ^2 \right] +\frac{\gamma ^2Ld}{2n}\tilde{\sigma}^2\tilde{\sigma}_{k}^{2}
\end{aligned}
\end{equation}
with
\begin{equation*}
A_1=-\gamma \left< \nabla f\left( \bar{x}^k \right) ,\frac{1}{n}\sum_{i=1}^n{\nabla f_i\left( z_{i}^{k} \right)} \right> 
\end{equation*}
and 
\begin{equation*}
\begin{aligned}
& A_2=
\\
&-\gamma \left< \nabla f\left( \bar{x}^k \right) ,\frac{1}{n}\sum_{i=1}^n{\sum_{j=1}^J{\frac{1}{J}\mathbb{I}\left[ \left\| \nabla f_i\left( z_{i}^{k};j \right) \right\| >C_k \right] \cdot}}\nabla f_i\left( z_{i}^{k};j \right) \cdot \left( \frac{C_k}{\left\| \nabla f_i\left( z_{i}^{k}; j \right) \right\|}-1 \right) \right> .
\end{aligned}
\end{equation*}

For $A_1$, we can bound it by
\begin{equation}
\label{A_1}
\begin{aligned}
A_1=& -\frac{\gamma}{2}\left\| \nabla f\left( \bar{x}^k \right) \right\| ^2-\frac{\gamma}{2}\left\| \frac{1}{n}\sum_{i=1}^n{\nabla f_i\left( z_{i}^{k} \right)} \right\| ^2
+\frac{\gamma}{2}\left\| \frac{1}{n}\sum_{i=1}^n{\nabla f_i\left( z_{i}^{k} \right)}-\nabla f\left( \bar{x}^k \right) \right\| ^2
\\
\leqslant & -\frac{\gamma}{2}\left\| \nabla f\left( \bar{x}^k \right) \right\| ^2+\frac{\gamma}{2}\left\| \frac{1}{n}\sum_{i=1}^n{\left( \nabla f_i\left( z_{i}^{k} \right) -\nabla f_i\left( \bar{x}^k \right) \right)} \right\| ^2
\\
\leqslant & -\frac{\gamma}{2}\left\| \nabla f\left( \bar{x}^k \right) \right\| ^2+\frac{\gamma}{2n}\sum_{i=1}^n{\left\| \nabla f_i\left( z_{i}^{k} \right) -\nabla f_i\left( \bar{x}^k \right) \right\| ^2}
\\
\leqslant & -\frac{\gamma}{2}\left\| \nabla f\left( \bar{x}^k \right) \right\| ^2+\frac{\gamma L^2}{2n}\sum_{i=1}^n{\left\| z_{i}^{k}-\bar{x}^k \right\| ^2},
\end{aligned}
\end{equation}
where we have used Assumption~\ref{assumption_smooth} in the last inequality.

Using Cauchy-Schwartz inequality, we can bound $A_2$ by
\begin{equation}
\label{A_2}
\begin{aligned}
& A_2
\\
\leqslant & \gamma \left\| \nabla f\left( \bar{x}^k \right) \right\| \left\| \frac{1}{n}\sum_{i=1}^n{\sum_{j=1}^J{\frac{1}{J}\mathbb{I}\left[ \left\| \nabla f_i\left( z_{i}^{k};j \right) \right\| >C_k \right] \cdot \nabla f_i\left( z_{i}^{k};j \right) \cdot \left( \frac{C_k}{\left\| \nabla f_i\left( z_{i}^{k}; j \right) \right\|}-1 \right)}} \right\| 
\\
\leqslant & \gamma \left\| \nabla f\left( \bar{x}^k \right) \right\| \cdot \frac{1}{n}\sum_{i=1}^n{\sum_{j=1}^J{\frac{1}{J}\mathbb{I}\left[ \left\| \nabla f_i\left( z_{i}^{k};j \right) \right\| >C_k \right] \cdot \varLambda}}
\\
= & \gamma \varLambda \left\| \nabla f\left( \bar{x}^k \right) \right\| \cdot \frac{1}{n}\sum_{i=1}^n{\mathcal{P}_{i}^{k}\left( C_k \right)},
\end{aligned}
\end{equation}
where $\mathcal{P}_{i}^{k}\left( C_k \right) $ denotes the probability of a per-sample gradient being clipped with threshold $C_k$ for node $i$ at iteration $k$.

As a result, substituting~\eqref{A_1} and~\eqref{A_2} into~\eqref{eq_3}, we further have
\begin{equation}
\label{eq_4}
\begin{aligned}
\mathbb{E}_k\left[ f\left( \bar{x}^{k+1} \right) \right] 
\leqslant & f\left( \bar{x}^k \right) -\frac{\gamma}{2}\left\| \nabla f\left( \bar{x}^k \right) \right\| ^2+\frac{\gamma L^2}{2n}\sum_{i=1}^n{\left\| z_{i}^{k}-\bar{x}^k \right\| ^2}
\\
&+\gamma \varLambda \left\| \nabla f\left( \bar{x}^k \right) \right\| \cdot \frac{1}{n}\sum_{i=1}^n{\mathcal{P}_{i}^{k}\left( C_k \right)}
\\
&+\frac{\gamma ^2L}{2}\mathbb{E}_k\left[ \left\| \frac{1}{n}\sum_{i=1}^n{g_{i}^{k}} \right\| ^2 \right] +\frac{\gamma ^2Ld}{2n}\tilde{\sigma}^2\tilde{\sigma}_{k}^{2} .
\end{aligned}
\end{equation}

Taking the total expectation on both sides of~\eqref{eq_4} yields
\begin{equation}
\begin{aligned}
\mathbb{E}\left[ f\left( \bar{x}^{k+1} \right) \right] 
\leqslant & \mathbb{E}\left[ f\left( \bar{x}^k \right) \right] -\frac{\gamma}{2}\mathbb{E}\left[ \left\| \nabla f\left( \bar{x}^k \right) \right\| ^2 \right] 
\\
&+\frac{\gamma L^2}{2n}\sum_{i=1}^n{\mathbb{E}\left[ \left\| z_{i}^{k}-\bar{x}^k \right\| ^2 \right]}+\frac{\gamma ^2L}{2}\mathbb{E}\left[ \left\| \frac{1}{n}\sum_{i=1}^n{g_{i}^{k}} \right\| ^2 \right] 
\\
&+\frac{\gamma ^2Ld}{2n}\tilde{\sigma}^2\tilde{\sigma}_{k}^{2}+\gamma \mathbb{E}\left[ \varLambda \left\| \nabla f\left( \bar{x}^k \right) \right\| \cdot \frac{1}{n}\sum_{i=1}^n{\mathcal{P}_{i}^{k}\left( C_k \right)} \right] .
\end{aligned}
\end{equation}

Summing the above inequality from $k=0$ to $K-1$, we obtain 
\begin{equation}
\label{ineq_2}
\begin{aligned}
\frac{\gamma}{2}\sum_{k=0}^{K-1}{\mathbb{E}\left[ \left\| \nabla f\left( \bar{x}^k \right) \right\| ^2 \right]}
\leqslant & f\left( \bar{x}^0 \right) -f^*+\frac{\gamma L^2}{2}\sum_{k=0}^{K-1}{\frac{1}{n}\sum_{i=1}^n{\mathbb{E}\left[ \left\| z_{i}^{k}-\bar{x}^k \right\| ^2 \right]}}
\\
& +\frac{\gamma ^2L}{2}\sum_{k=0}^{K-1}{\mathbb{E}\left[ \left\| \frac{1}{n}\sum_{i=1}^n{g_{i}^{k}} \right\| ^2 \right]}+\frac{\gamma ^2Ld}{2n}\tilde{\sigma}^2\sum_{k=0}^{K-1}{\tilde{\sigma}_{k}^{2}}
\\
& +\gamma \mathbb{E}\left[ \sum_{k=0}^{K-1}{\varLambda \left\| \nabla f\left( \bar{x}^k \right) \right\| \cdot \frac{1}{n}\sum_{i=1}^n{\mathcal{P}_{i}^{k}\left( C_k \right)}} \right] .
\end{aligned}
\end{equation}

Substituting the upper bound of $\sum_{k=0}^{K-1}{\frac{1}{n}\sum_{i=1}^n{\mathbb{E}\left[ \left\| z_{i}^{k}-\bar{x}^k \right\| ^2 \right]}}$ (c.f.,~\eqref{upper_bound_sum_of_M_k} in Lemma~\ref{Lemma_consensus_error}) into~\eqref{ineq_2}, and multiplying $\frac{2}{\gamma K}$ on both sides, we obtain
\begin{equation}
\label{final_ineq}
\begin{aligned}
& \frac{1}{K}\sum_{k=0}^{K-1}{\mathbb{E}\left[ \left\| \nabla f\left( \bar{x}^k \right) \right\| ^2 \right]}
\\
\leqslant & \frac{2\left( f\left( \bar{x}^0 \right) -f^* \right)}{\gamma K}+\frac{3L^2 \varPsi^2\sum_{i=1}^n{\left\| x_{i}^{0} \right\| ^2}}{\left( 1-q \right) ^2nK}
+\frac{3\gamma ^2L^2 \varPsi^2}{\left( 1-q \right) ^2}\cdot \frac{1}{K}\sum_{k=0}^{K-1}{\cdot \frac{1}{n}\sum_{i=1}^n{\mathbb{E}\left[ \left\| g_{i}^{k} \right\| ^2 \right]}}
\\
& +\gamma L\cdot \frac{1}{K}\sum_{k=0}^{K-1}{\mathbb{E}\left[ \left\| \frac{1}{n}\sum_{i=1}^n{g_{i}^{k}} \right\| ^2 \right]}
+\underset{\text{$\triangleq T_1$: privacy noise term}}{\underbrace{\frac{3\gamma ^2L^2 \varPsi ^2d}{\left( 1-q \right) ^2}\tilde{\sigma}^2\cdot \frac{1}{K}\sum_{k=0}^{K-1}{\tilde{\sigma}_{k}^{2}}+\frac{\gamma Ld}{n}\tilde{\sigma}^2\cdot \frac{1}{K}\sum_{k=0}^{K-1}{\tilde{\sigma}_{k}^{2}}}}
\\
& +\underset{\text{$\triangleq T_2$: bias term}}{\underbrace{2\mathbb{E}\left[ \frac{1}{K}\sum_{k=0}^{K-1}{\varLambda \left\| \nabla f\left( \bar{x}^k \right) \right\| \cdot \frac{1}{n}\sum_{i=1}^n{P_{i}^{k}\left( C_k \right)}} \right] }}.
\end{aligned}
\end{equation}

Substituting $\tilde{\sigma} = \frac{1}{J\mu _{tot}}\sqrt{2\sum_{k=0}^{K-1}{\frac{C_{k}^{2}}{\tilde{\sigma}_{k}^{2}}}}$ from Proposition~\ref{lemma_of_noise_scale} into the above~\eqref{final_ineq}, we can further obtain
\begin{equation}
\label{some_ineq}
\begin{aligned}
& \frac{1}{K}\sum_{k=0}^{K-1}{\mathbb{E}\left[ \left\| \nabla f\left( \bar{x}^k \right) \right\| ^2 \right]}
\\
\leqslant & \frac{2\left( f\left( \bar{x}^0 \right) -f^* \right)}{\gamma K}+\frac{3L^2 \varPsi^2\sum_{i=1}^n{\left\| x_{i}^{0} \right\| ^2}}{\left( 1-q \right) ^2nK} +\frac{3\gamma ^2L^2 \varPsi^2}{\left( 1-q \right) ^2}\cdot \frac{1}{K}\sum_{k=0}^{K-1}{\cdot \frac{1}{n}\sum_{i=1}^n{\mathbb{E}\left[ \left\| g_{i}^{k} \right\| ^2 \right]}}
\\
& +\gamma L\cdot \frac{1}{K}\sum_{k=0}^{K-1}{\mathbb{E}\left[ \left\| \frac{1}{n}\sum_{i=1}^n{g_{i}^{k}} \right\| ^2 \right]} +\left( \frac{6\gamma ^2KL^2\varPsi ^2d}{\left( 1-q \right) ^2J^2\mu _{tot}^{2}}+\frac{2\gamma KLd}{nJ^2\mu _{tot}^{2}} \right) \cdot \underset{\text{$\triangleq T_1$: privacy noise term}}{\underbrace{\frac{1}{K}\sum_{k=0}^{K-1}{\frac{C_{k}^{2}}{\tilde{\sigma}_{k}^{2}}}\cdot \frac{1}{K}\sum_{k=0}^{K-1}{\tilde{\sigma}_{k}^{2}}}}
\\
&+\underset{\text{$\triangleq T_2$: bias term}}{\underbrace{2\mathbb{E}\left[ \frac{1}{K}\sum_{k=0}^{K-1}{\varLambda \left\| \nabla f\left( \bar{x}^k \right) \right\| \cdot \frac{1}{n}\sum_{i=1}^n{\mathcal{P}_{i}^{k}\left( C_k \right)}} \right] }}
,
\end{aligned}
\end{equation}
which completes the proof of Theorem~\ref{main_Theorem}.

\section{Proof of Corollary~\ref{Corollary_after_theorem}}
\label{proof_of_corollary}

\begin{proof}
Substituting $C_k=\Theta \left( \left( \rho _c \right) ^{-\frac{k}{K}} \right) 
$ and $\tilde{\sigma}_k=\Theta \left( \left( \rho _c\cdot \rho _{\mu} \right) ^{-\frac{k}{K}} \right)$ into $T_1$ in~\eqref{some_ineq}, we have
\begin{equation}
\label{involve_1}
\begin{aligned}
T_1&=\Theta \left( \frac{1}{K}\sum_{k=0}^{K-1}{\frac{\left( \rho _{c}^{2} \right) ^{-\frac{k}{K}}}{\left( \rho _{c}^{2}\cdot \rho _{\mu}^{2} \right) ^{-\frac{k}{K}}}} \right) \cdot \Theta \left( \frac{1}{K}\sum_{k=0}^{K-1}{\left( \rho _{c}^{2}\cdot \rho _{\mu}^{2} \right) ^{-\frac{k}{K}}} \right) 
\\
&=\Theta \left( \frac{1}{K}\sum_{k=0}^{K-1}{\left( \rho _{\mu}^{2} \right) ^{\frac{k}{K}}} \right) \cdot \Theta \left( \frac{1}{K}\sum_{k=0}^{K-1}{\left( \rho _{c}^{2}\cdot \rho _{\mu}^{2} \right) ^{-\frac{k}{K}}} \right) 
\\
&=\Theta \left( \frac{1}{K}\cdot \frac{\rho _{\mu}^{2}-1}{\left( \rho _{\mu}^{2} \right) ^{\frac{1}{K}}-1} \right) \cdot \Theta \left( \frac{1}{K}\cdot \frac{1-\rho _{c}^{-2}\rho _{\mu}^{-2}}{1-\left( \rho _{c}^{-2}\rho _{\mu}^{-2} \right) ^{\frac{1}{K}}} \right) .
\end{aligned}
\end{equation}

According to the fact that
$\frac{1}{K}\cdot \frac{\rho _{\mu}^{2}-1}{\left( \rho _{\mu}^{2} \right) ^{\frac{1}{K}}-1}<\frac{\rho _{\mu}^{2}-1}{\ln \left( \rho _{\mu}^{2} \right)}$ 
and $\frac{1}{K}\cdot \frac{1-\rho _{c}^{-2}\rho _{\mu}^{-2}}{1-\left( \rho _{c}^{-2}\rho _{\mu}^{-2} \right) ^{\frac{1}{K}}}\leqslant 1$ for all $K\geqslant 1
$, we have
\begin{equation}
\label{involve_2}
T_1=\mathcal{O}\left( 1 \right) .
\end{equation}

Therefore,~\eqref{some_ineq} becomes
\begin{equation}
\label{gyb}
\begin{aligned}
& \frac{1}{K}\sum_{k=0}^{K-1}{\mathbb{E}\left[ \left\| \nabla f\left( \bar{x}^k \right) \right\| ^2 \right]}
\\
\leqslant & \mathcal{O}\left( \frac{1}{\gamma K} \right) +\frac{1}{\left( 1-q \right) ^2}\cdot \mathcal{O}\left( \frac{1}{K} \right) +\frac{1}{\left( 1-q \right) ^2}\cdot \mathcal{O}\left( \gamma ^2 \right) 
\\
& +\mathcal{O}\left( \gamma \right) +\frac{1}{\left( 1-q \right) ^2}\cdot \mathcal{O}\left( \frac{\gamma ^2K}{J^2\mu _{tot}^{2}} \right) +\mathcal{O}\left( \frac{\gamma K}{nJ^2\mu _{tot}^{2}} \right) +T_2.
\end{aligned}
\end{equation}

Substituting $\gamma =\frac{1}{\sqrt{n}J\mu _{tot}}$ and $\gamma K=\sqrt{n}J\mu _{tot}$ into~\eqref{gyb}, and under our mild assumption of $J\mu _{tot}>\sqrt{n}$, we obtain
\begin{equation}
\frac{1}{K}\sum_{k=0}^{K-1}{\mathbb{E}\left[ \left\| \nabla f\left( \bar{x}^k \right) \right\| ^2 \right]}\leqslant \mathcal{O}\left( \frac{1}{\left( 1-q \right) ^2\sqrt{n}J\mu _{tot}} \right) +T_2,
\end{equation}
which completes the proof of Corollary~\ref{Corollary_after_theorem}.
\end{proof}

\section{Proof of Proposition~\ref{lemma_of_noise_scale}}
\label{proof_of_pro}

\begin{proof}
According to the definition of GDP in Proposition~\ref{formal_def_of_mu}, we can easily know that each node satisfies $\mu_k$-GDP for each iteration in Algorithm~\ref{general_Adpt_PrivSGP_Combine}, where
\begin{equation*}
\mu _k=\frac{C_k}{\tilde{\sigma}\cdot \tilde{\sigma}_k}.
\end{equation*}
Further, based on~\eqref{mu_total}, we can derive that each node satisfies $\hat{\mu}_{tot}$-GDP for overall training process, with
\begin{equation*}
\hat{\mu}_{tot}=\frac{1}{J}\sqrt{\sum_{k=0}^{K-1}{\left( e^{\mu _{k}^{2}}-1 \right)}} \overset{\left( a \right)}{<} \frac{1}{J}\sqrt{2\sum_{k=0}^{K-1}{\mu _{k}^{2}}}
\end{equation*}
where in $(a)$ we used the fact that $e^x-1<2x$ for $0\leqslant x\leqslant 1$.
Given a privacy parameter $\mu_{tot}$ calculated according to $\left( \epsilon ,\delta \right) $ based on~\eqref{privacy_tranfer}, we use it to bound the above $\hat{\mu}_{tot}$, i.e.,
\begin{equation*}
\hat{\mu}_{tot}<\frac{1}{J}\sqrt{2\sum_{k=0}^{K-1}{\mu _{k}^{2}}}=\frac{1}{J}\sqrt{2\sum_{k=0}^{K-1}{\frac{C_{k}^{2}}{\tilde{\sigma}^2\cdot \tilde{\sigma}_{k}^{2}}}}\leqslant \mu _{tot},
\end{equation*}
and we obtain
\begin{equation*}
\tilde{\sigma}\geqslant \frac{1}{J\mu _{tot}}\sqrt{2\sum_{k=0}^{K-1}{\frac{C_{k}^{2}}{\tilde{\sigma}_{k}^{2}}}},
\end{equation*}
which is sufficient to guarantee each node in Algorithm~\ref{general_Adpt_PrivSGP_Combine} to satisfy $\mu _{tot}$-GDP as well as $\left( \epsilon,\delta \right)$-DP.
\end{proof}

\section{Missing Pseudocodes of Some Algorithms}
\label{missed_pseudocode}

In this section, we supplement the pseudocodes of {\algS}, {\algmu} and {\algbase} we missed in the main text, which are given in Algorithm~\ref{Adpt_PrivSGP}, Algorithm~\ref{Adpt_PrivSGP_mu} and Algorithm~\ref{Const_PrivSGP}, respectively.

\begin{algorithm}[H]
      \caption{{\algS}} 
      \label{Adpt_PrivSGP} 
        \begin{algorithmic}[1]
        \STATE \textbf{Initialization:}  DP budget $(\epsilon,\delta)$, $x_{i}^{0}=z_{i}^{0}\in \mathbb{R}^d$, $w_i^0=1$, step size $\gamma > 0$, total number of iterations $K$, initial clipping threshold $C_0$ and hyper-parameter $\rho_c$.
        \FOR{$k=0,1,...,K-1$, at node $i$,}
        \STATE Randomly samples a local training data $\xi_i^k$ with the sampling probability $\frac{1}{J}$;
        \STATE Computes stochastic gradient at $z_i^k$: $\nabla f_i(z_i^k;\xi_i^k)$;
        \STATE Calculates the clipping bound by: $C_k=C_0\cdot \left( \rho _c \right) ^{-\frac{k}{K}}$; 
        \STATE Clips the stochastic gradient: 
        \begin{equation*}
        g_{i}^{k}=\mathrm{Clip}\left( \nabla f_i\left( z_{i}^{k};\xi _{i}^{k} \right) ;C_k \right) =\nabla f_i\left( z_{i}^{k};\xi _{i}^{k} \right) \cdot \min \left( 1,\frac{C_k}{\left\| \nabla f_i\left( z_{i}^{k};\xi _{i}^{k} \right) \right\|} \right) ;
        \end{equation*}
        \STATE Calculates the DP noise variance: $\sigma _k=\frac{C_0}{\bar{\mu}}\cdot \left( \rho _c \right) ^{-\frac{k}{K}}$ ;
        \STATE Draws randomized noise $N_i^k$ from the Gaussian distribution: $N_{i}^{k}\sim \mathcal{N}\left( 0,\sigma_k^2\mathbb{I}_d \right)$;
        \STATE Differentially private local SGD: 
        \begin{equation*}
        x_{i}^{k+\frac{1}{2}}=x_{i}^{k}-\gamma (g_{i}^{k}+N_{i}^{k});
        \end{equation*}
        \STATE Sends $\left( x_i^{k+\frac{1}{2}}, w_i^k \right)$ to all out-neighbors and receives $\left( x_j^{k+\frac{1}{2}}, w_j^k \right)$ from all in-neighbors ;
        \STATE Updates $x_i^{k+1}$ by: \quad $x_{i}^{k+1}=\sum_{j=1}^n{P_{i,j}^{k}x_{j}^{k+\frac{1}{2}}}$ ;
        \STATE Updates $w_i^{k+1}$ by: \quad $w_{i}^{k+1}=\sum_{j=1}^n{P_{i,j}^{k}w_{j}^{k}}$ ;
        \STATE Updates $z_i^{k+1}$ by: \quad     
        $z_{i}^{k+1}=x_{i}^{k+1}/w_i^{k+1}$ .
        \ENDFOR
    \end{algorithmic}
\end{algorithm}

\begin{algorithm}[H]
      \caption{{\algmu}} 
      \label{Adpt_PrivSGP_mu} 
        \begin{algorithmic}[1]
        \STATE \textbf{Initialization:}  DP budget $(\epsilon,\delta)$, $x_{i}^{0}=z_{i}^{0}\in \mathbb{R}^d$, $w_i^0=1$, step size $\gamma > 0$, total number of iterations $K$, fixed clipping threshold $\bar{C}$ and hyper-parameter $\rho_\mu$.
        \FOR{$k=0,1,...,K-1$, at node $i$,}
        \STATE Randomly samples a local training data $\xi_i^k$ with the sampling probability $\frac{1}{J}$;
        \STATE Computes stochastic gradient at $z_i^k$: $\nabla f_i(z_i^k;\xi_i^k)$;
        \STATE Clips the stochastic gradient: 
        \begin{equation*}
        g_{i}^{k}=\mathrm{Clip}\left( \nabla f_i\left( z_{i}^{k};\xi _{i}^{k} \right) ;\bar{C} \right) =\nabla f_i\left( z_{i}^{k};\xi _{i}^{k} \right) \cdot \min \left\{ 1,\frac{\bar{C}}{\left\| \nabla f_i(z_{i}^{k};\xi _{i}^{k}) \right\|} \right\} ;
        \end{equation*}
        \STATE Calculates the DP noise variance: $\sigma _k=\frac{\bar{C}}{\mu _0}\cdot \left( \rho _{\mu} \right) ^{-\frac{k}{K}}$ ;
        \STATE Draws randomized noise $N_i^k$ from the Gaussian distribution: $N_{i}^{k}\sim \mathcal{N}\left( 0,\sigma_k^2\mathbb{I}_d \right)$ ;
        \STATE Differentially private local SGD:
        \begin{equation*}
        x_{i}^{k+\frac{1}{2}}=x_{i}^{k}-\gamma (g_{i}^{k}+N_{i}^{k}) ;
        \end{equation*}
        \STATE Sends $\left( x_i^{k+\frac{1}{2}}, w_i^k \right)$ to all out-neighbors and receives $\left( x_j^{k+\frac{1}{2}}, w_j^k \right)$ from all in-neighbors ;
        \STATE Updates $x_i^{k+1}$ by: \quad $x_{i}^{k+1}=\sum_{j=1}^n{P_{i,j}^{k}x_{j}^{k+\frac{1}{2}}}$ ;
        \STATE Updates $w_i^{k+1}$ by: \quad $w_{i}^{k+1}=\sum_{j=1}^n{P_{i,j}^{k}w_{j}^{k}}$ ;
        \STATE Updates $z_i^{k+1}$ by: \quad     
        $z_{i}^{k+1}=x_{i}^{k+1}/w_i^{k+1}$ .
        \ENDFOR
        \end{algorithmic}
\end{algorithm}

\begin{algorithm}[H]
      \caption{{\algbase}} 
      \label{Const_PrivSGP} 
        \begin{algorithmic}[1]
        \STATE \textbf{Initialization:}  DP budget $(\epsilon,\delta)$, $x_{i}^{0}=z_{i}^{0}\in \mathbb{R}^d$, $w_i^0=1$, step size $\gamma > 0$, total number of iterations $K$, and fixed clipping bound $\bar{C}$.
        \FOR{$k=0,1,...,K-1$, at node $i$,}
        \STATE Randomly samples a local training data $\xi_i^k$ with the sampling probability $\frac{1}{J}$;
        \STATE Computes stochastic gradient at $z_i^k$: $\nabla f_i(z_i^k;\xi_i^k)$;
        \STATE Clips the stochastic gradient:
        \begin{equation}
        g_{i}^{k}=\mathrm{Clip}\left( \nabla f_i\left( z_{i}^{k};\xi _{i}^{k} \right) ;\bar{C} \right) =\nabla f_i\left( z_{i}^{k};\xi _{i}^{k} \right) \cdot \min \left\{ 1,\frac{\bar{C}}{\left\| \nabla f_i(z_{i}^{k};\xi _{i}^{k}) \right\|} \right\} ;
        \end{equation}
        \STATE Draws randomized noise $N_i^k$ from the Gaussian distribution: 
        \begin{equation*}
        N_{i}^{k}\sim \mathcal{N}\left( 0, \bar{\sigma}^2\mathbb{I}_d \right),
        \end{equation*}
        where $\bar{\sigma}=\frac{\bar{C}}{\bar{\mu}}$;
        \STATE Differentially private local SGD:
        \begin{equation*}
        x_{i}^{k+\frac{1}{2}}=x_{i}^{k}-\gamma (g_{i}^{k}+N_{i}^{k}) ;
        \end{equation*}
        \STATE Sends $\left( x_i^{k+\frac{1}{2}}, w_i^k \right)$ to all out-neighbors and receives $\left( x_j^{k+\frac{1}{2}}, w_j^k \right)$ from all in-neighbors ;
        \STATE Updates $x_i^{k+1}$ by: \quad $x_{i}^{k+1}=\sum_{j=1}^n{P_{i,j}^{k}x_{j}^{k+\frac{1}{2}}}$ ;
        \STATE Updates $w_i^{k+1}$ by: \quad $w_{i}^{k+1}=\sum_{j=1}^n{P_{i,j}^{k}w_{j}^{k}}$ ;
        \STATE Updates $z_i^{k+1}$ by: \quad     
        $z_{i}^{k+1}=x_{i}^{k+1}/w_i^{k+1}$ .
        \ENDFOR
    \end{algorithmic}
\end{algorithm}

\section{Missing Definition of Time-varying Directed Exponential Graph}
\label{missing_definition_of_graph}

We supplement the definition of time-varying directed exponential graph~\cite{assran2019stochastic} we missed in the main text. Specifically, $n$ nodes are ordered sequentially with their rank $0$,$1$...,$n-1$, and each node has out-neighbours that are $2^0$,$2^1$,...,$2^{\lfloor \log _2\left( n-1 \right) \rfloor}$ hops away. Each node cycles through these out-neighbours, and only transmits messages to one of its out-neighbours at each iteration. For example, each node sends message to its $2^0$-hop out-neighbour at iteration $k$, and to its $2^1$-hop out-neighbour at iteration $k+1$, and so on. The above procedure will be repeated within the list of out-neighbours. 
Note that each node only sends and receives a single message at each iteration.

\section{Missing Experimental Setup}

\subsection{Hyper-parameter}
\label{experimental_set_up_parameter}
We supplement the setup about the hyper-parameters we missed in the main text.
Constant clipping bound $\bar{C}$, initial clipping bound $C_0$, $\rho_c$ and $\rho_\mu$ are given as follows:
\begin{itemize}
    \item For {\algbase} (c.f., Algorithm~\ref{Const_PrivSGP}), we test different value of constant clipping bound $\bar{C}$ from the set $\{5, 4, 3, 2, 1\}$ (resp. $\{2.5, 2, 1.5, 1, 0.5\}$) for ResNet-18 training (resp. shallow CNN model training), and choose the one with the best model performance for comparison with other algorithms;
    \item For {\algS} (c.f., Algorithm~\ref{Adpt_PrivSGP}), we set the initial clipping bound $C_0$ to $7$ (resp. $4$) for ResNet-18 training (resp. shallow CNN model training), test different value of $\rho_c$ in the set $1/\rho_c \in \{0.1, 0.2, 0.3, 0.4, 0.5, 0.6, 0.7, 0.8\}$ and choose the one with the best model performance for comparison with other algorithms;
    \item For {\algmu} (c.f., Algorithm~\ref{Adpt_PrivSGP_mu}), we search the constant clipping bound $\bar{C}$ from the set $\{5, 4, 3, 2, 1\}$ (resp. $\{2.5, 2, 1.5, 1, 0.5\}$) for ResNet-18 training (resp. shallow CNN model training), test different value of $\rho_\mu$ in the set $1/\rho_\mu \in \{0.1, 0.2, 0.3, 0.4, 0.5, 0.6, 0.7, 0.8\}$, and choose such a combination ($\bar{C}$ and $\rho_\mu$ ) that achieves the best model performance for comparison with other algorithms;
    \item  For {\alg} (c.f., Algorithm~\ref{Adpt_PrivSGP_combine}), we set the initial clipping bound $C_0$ to $7$ (resp. $4$) for ResNet-18 training (resp. shallow CNN model training), search $\rho_c$ and $\rho_\mu$ in the sets $1/\rho_c, 1/\rho_\mu \in \{0.1, 0.2, 0.3, 0.4, 0.5, 0.6, 0.7, 0.8\}$ and choose the combination with the best model performance for comparison with other algorithms.
\end{itemize}

\subsection{The number of training samples for each node}
\label{missing_experimental_set_up_node_number}
In this part, we supplement the experimental setup of experiments over different node numbers we missed in the main text.
The 50,000 training images of the Cifar-10 dataset are evenly divided into 40 nodes, so each node has 1250 training images, i.e., the number of training samples $J=1250$. In the four experiments with different numbers of nodes, the number of training samples for each node was kept consistent.

\section{Missing Experimental Results}
\label{missed_plots}
In this section, we supplement the experimental results we missed in the main text.

Figure~\ref{refer_cifar_1}-\ref{refer_cifar_2} show the plots of training loss/testing accuracy versus iteration for different algorithms when training ResNet-18 on Cifar-10 dataset, under the different values of privacy budget $\epsilon$.

\begin{figure}[!htpb]
  \centering
  \subfigure[$\epsilon=3, \delta=10^{-4}$]{
    \includegraphics[width=0.473\linewidth]{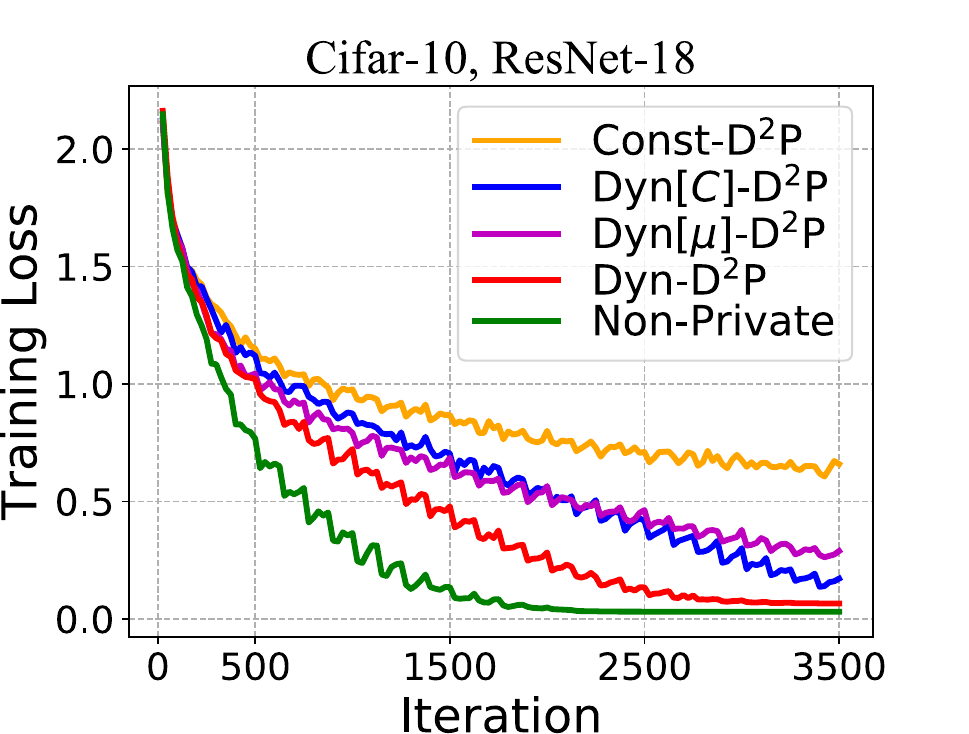}
  }
  \subfigure[$\epsilon=3, \delta=10^{-4}$]{
    \includegraphics[width=0.473\linewidth]{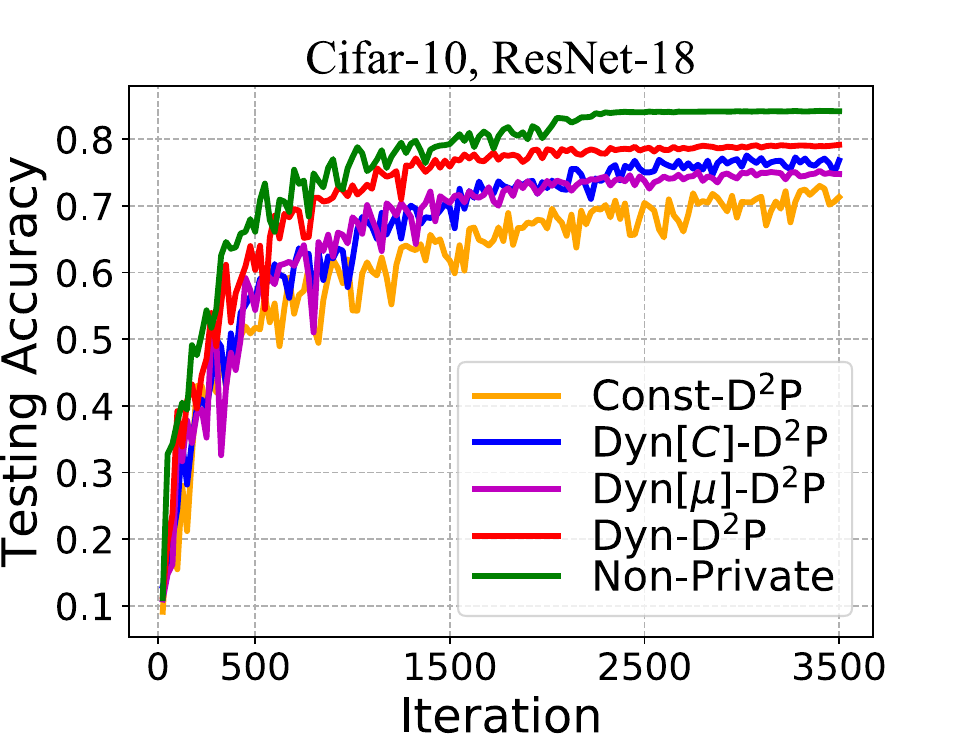}
  }
  \caption{Comparison of convergence performance for five algorithms under $(3, 10^{-4})$-DP guarantee for each node, when training ResNet-18 on Cifar-10 dataset.}
\label{refer_cifar_1}
\end{figure}

\begin{figure}[!htpb]
  \centering
  \subfigure[$\epsilon=1, \delta=10^{-4}$]{
    \includegraphics[width=0.473\linewidth]{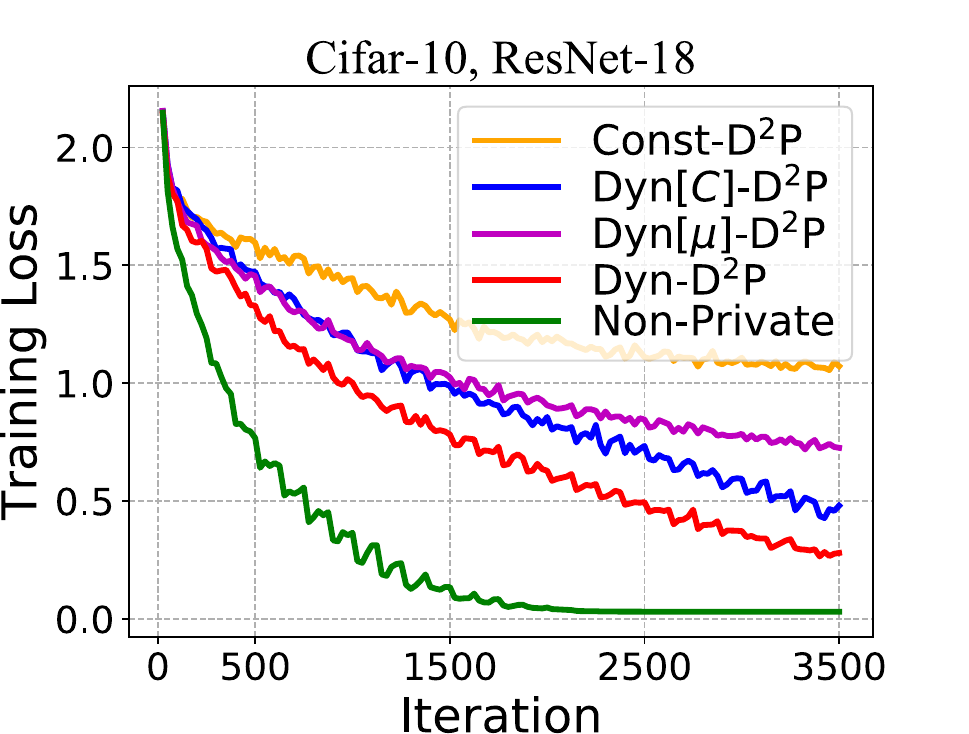}
  }
  \subfigure[$\epsilon=1, \delta=10^{-4}$]{
    \includegraphics[width=0.473\linewidth]{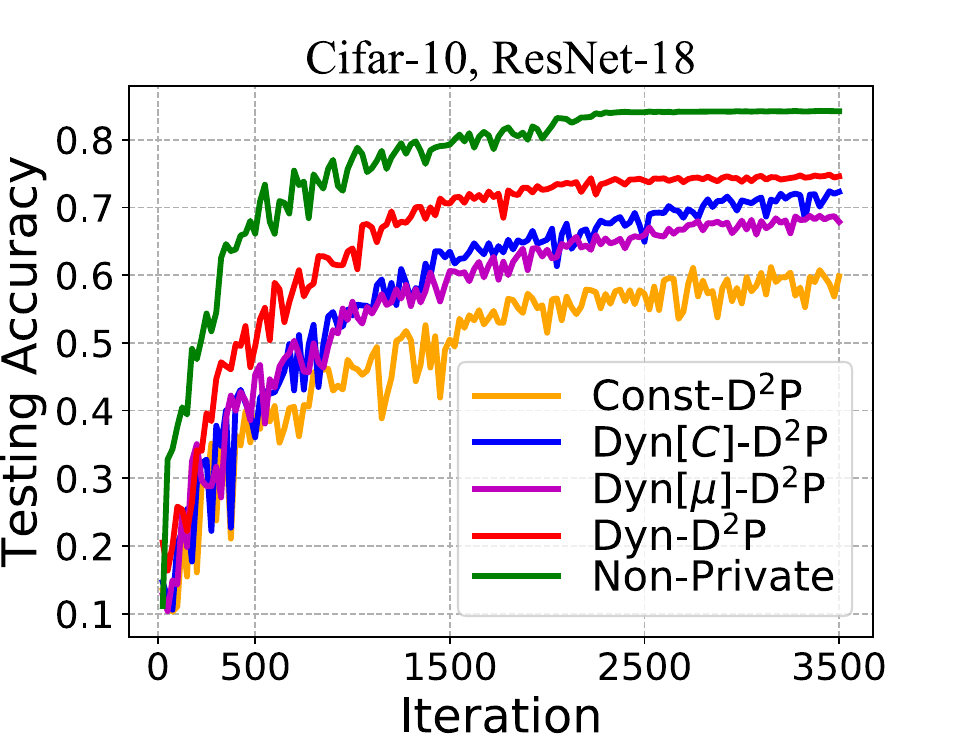}
  }
  \caption{Comparison of convergence performance for five algorithms under $(1, 10^{-4})$-DP guarantee for each node, when training ResNet-18 on Cifar-10 dataset.}
\end{figure}

\begin{figure}[!htpb]
  \centering
  \subfigure[$\epsilon=0.3, \delta=10^{-4}$]{
    \includegraphics[width=0.473\linewidth]{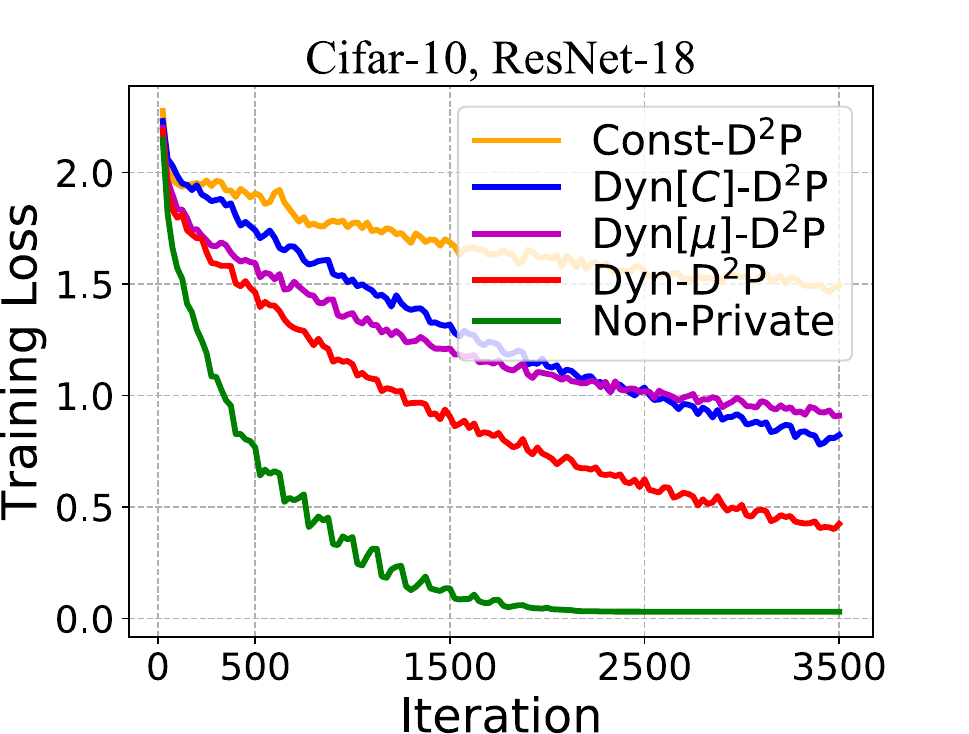}
  }
  \subfigure[$\epsilon=0.3, \delta=10^{-4}$]{
    \includegraphics[width=0.473\linewidth]{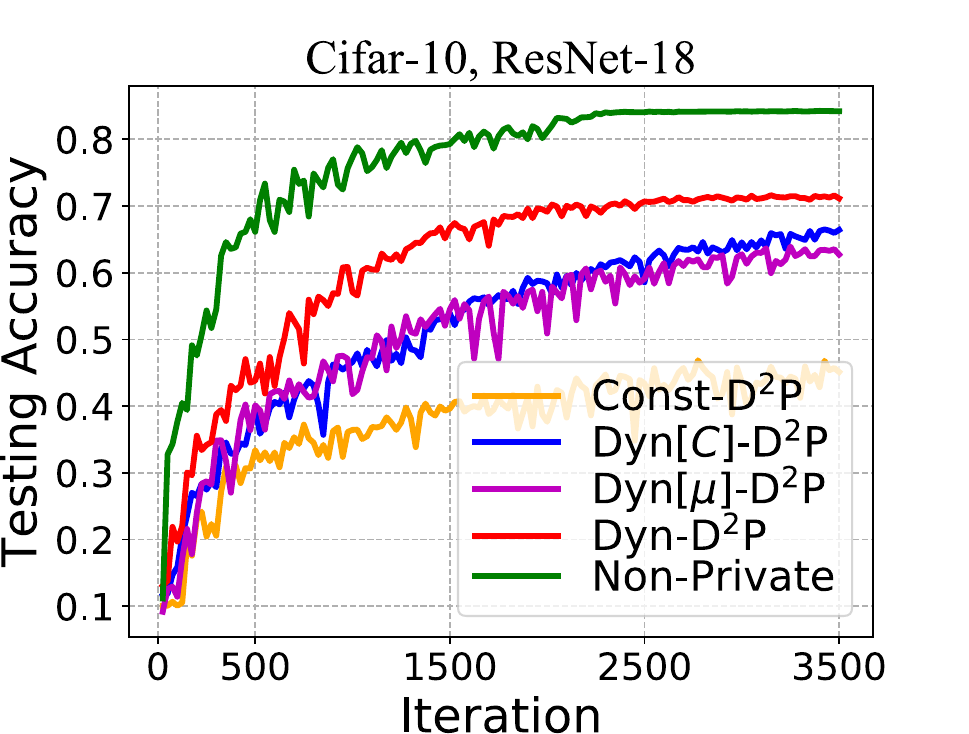}
  }
  \caption{Comparison of convergence performance for five algorithms under $(0.3, 10^{-4})$-DP guarantee for each node, when training ResNet-18 on Cifar-10 dataset.}
\label{refer_cifar_2}
\end{figure}


For the experiments of training shallow CNN on FashionMnist dataset, we present the plots of training loss/testing accuracy versus iteration in Figure~\ref{refer_mnist_1},~\ref{refer_mnist_2} and~\ref{refer_mnist_3}, with the different values of privacy budget $\epsilon$. It can be observed that our {\alg} and two by-product algorithms ({\algS} and {\algmu}) consistently outperform the baseline algorithm {\algbase} which employs constant noise. Among these above DP algorithms, {\alg} achieves the highest model accuracy while maintaining the same level of privacy protection. Furthermore, a comparison of experimental results with different privacy budgets (i.e., varying values of $\epsilon$) shows that the stronger the level of required privacy protection (i.e., the smaller the value of budget $\epsilon$), the more pronounced the advantage in model accuracy with our dynamic noise strategy. In particular, when setting a small $\epsilon=0.3$ which implies a strong privacy guarantee, our {\alg} achieves a $39\%$ higher model accuracy than {\algbase} employing constant noise strategies. These results validate the superiority of our dynamic noise approach.

\begin{figure}[!htpb]
  \centering
  \subfigure[Loss versus iteration]{
    \includegraphics[width=0.473\linewidth]{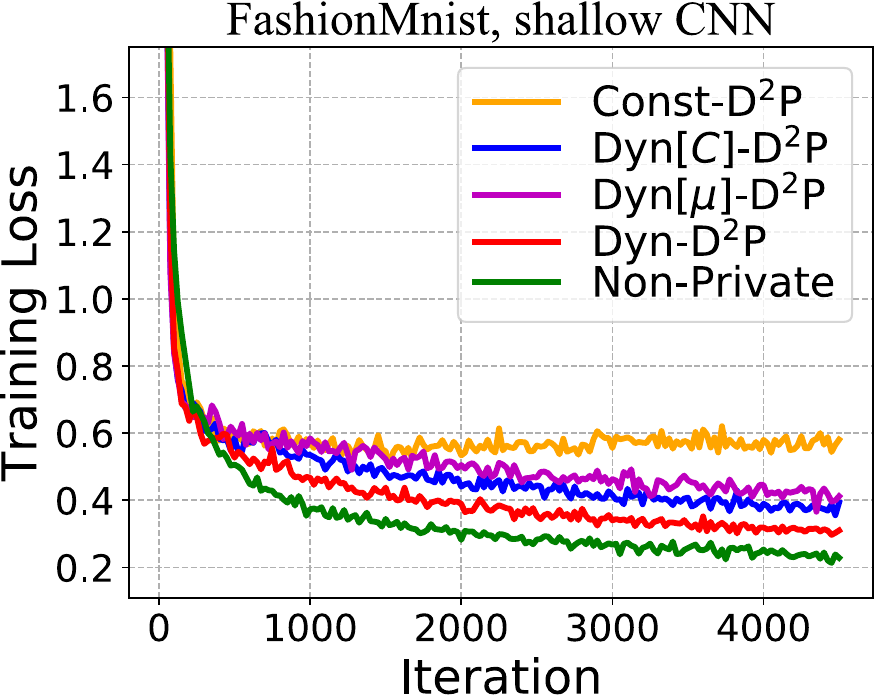}
  }
  \subfigure[Accuracy versus iteration]{
    \includegraphics[width=0.473\linewidth]{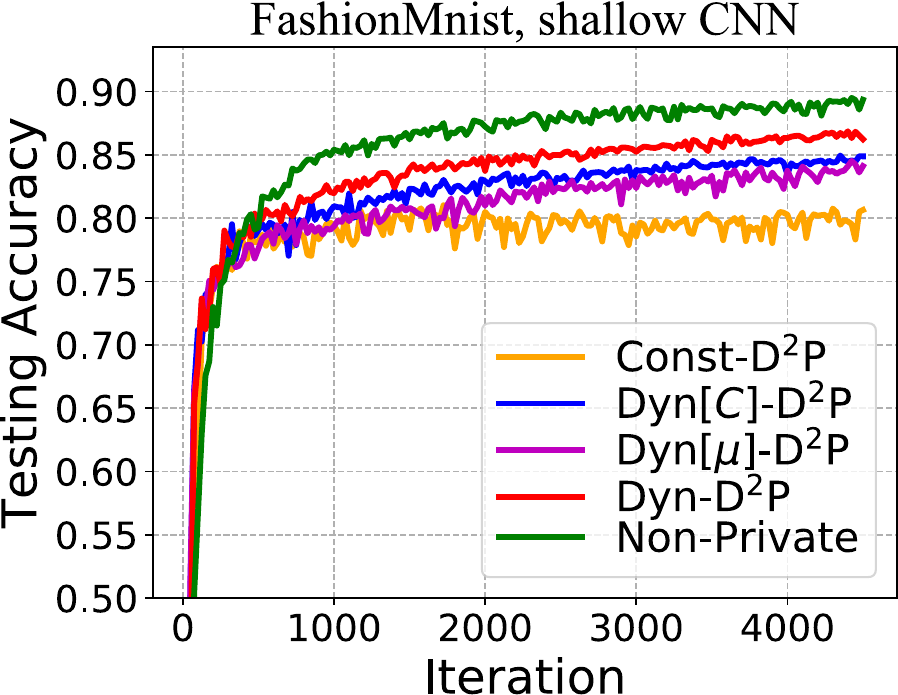}
  }
  \caption{Comparison of convergence performance for five algorithms under $(3, 10^{-4})$-DP guarantee for each node, when training shallow CNN on FashionMnist dataset.}
\label{refer_mnist_1}
\end{figure}

\begin{figure}[!htpb]
  \centering
  \subfigure[$\epsilon=1, \delta=10^{-4}$]{
    \includegraphics[width=0.473\linewidth]{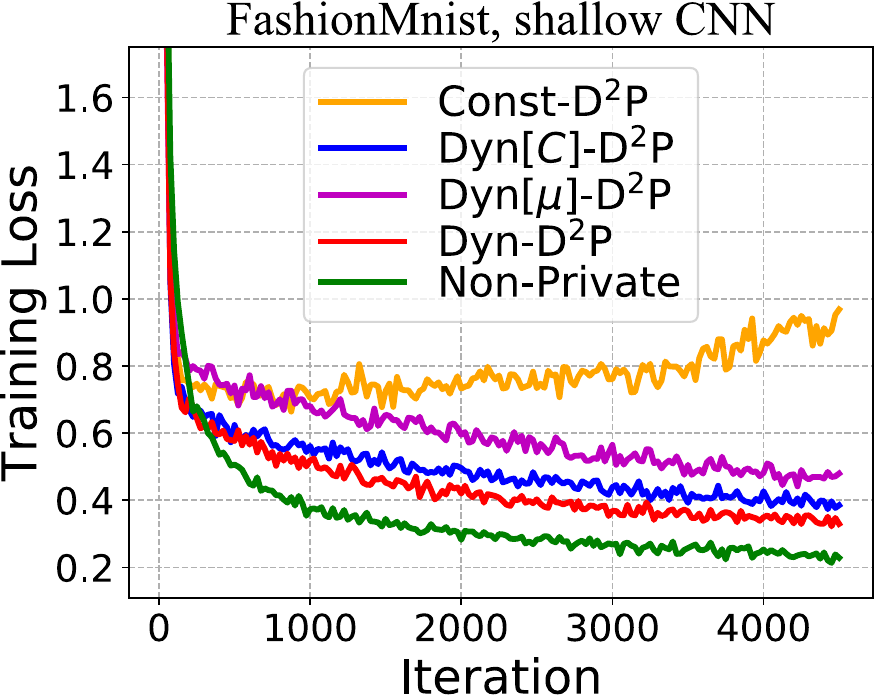}
  }
  \subfigure[$\epsilon=1, \delta=10^{-4}$]{
    \includegraphics[width=0.473\linewidth]{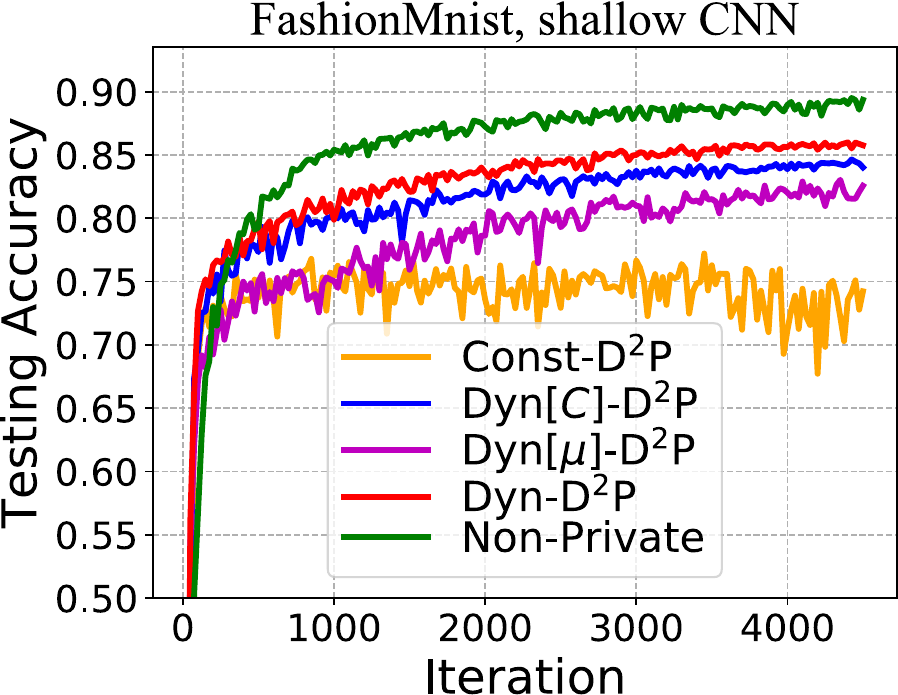}
  }
  \caption{Comparison of convergence performance for five algorithms under $(1, 10^{-4})$-DP guarantee for each node, when training shallow CNN on FashionMnist dataset.}
  \label{refer_mnist_2}
\end{figure}

\begin{figure}[!htpb]
  \centering
  \subfigure[$\epsilon=0.3, \delta=10^{-4}$]{
    \includegraphics[width=0.473\linewidth]{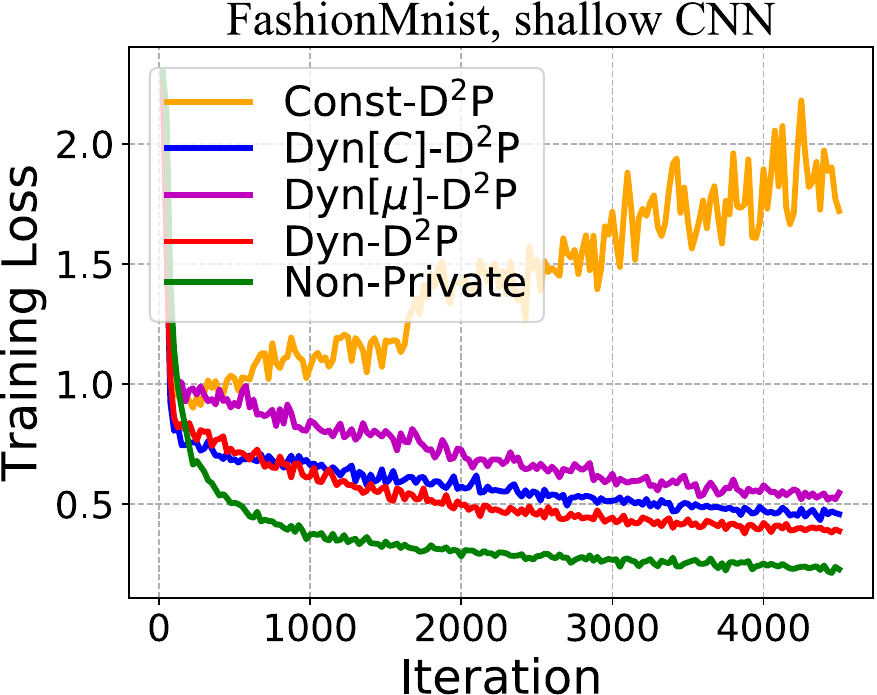}
  }
  \subfigure[$\epsilon=0.3, \delta=10^{-4}$]{
    \includegraphics[width=0.473\linewidth]{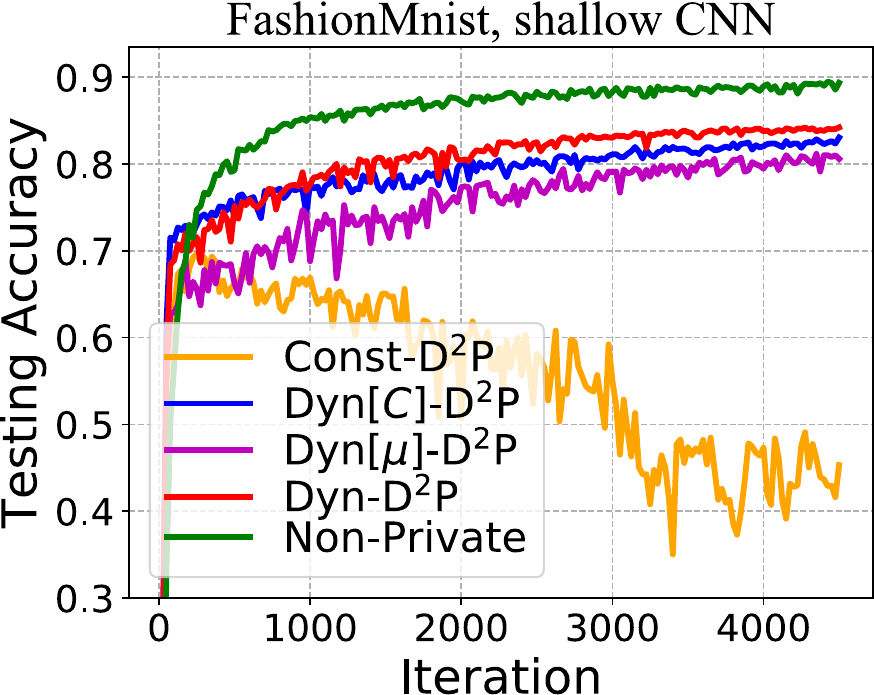}
  }
  \caption{Comparison of convergence performance for five algorithms under $(0.3, 10^{-4})$-DP guarantee for each node, when training shallow CNN on FashionMnist dataset.}
\label{refer_mnist_3}
\end{figure}

\end{document}